\documentclass{article}
\usepackage{iclr2027_conference,times}
\usepackage[utf8]{inputenc}
\usepackage[T1]{fontenc}
\usepackage{amsmath,amsfonts,bm}

\def\eqref#1{equation~\ref{#1}}
\def\1{\bm{1}}

\DeclareMathAlphabet{\mathsfit}{\encodingdefault}{\sfdefault}{m}{sl}
\SetMathAlphabet{\mathsfit}{bold}{\encodingdefault}{\sfdefault}{bx}{n}

 \usepackage{hyperref}
\hypersetup{colorlinks=true,allcolors=black,hypertexnames=false}
\usepackage{url}
\usepackage{graphicx}
\usepackage{booktabs}
\usepackage{multirow}
\usepackage{array}
\usepackage{amsmath}
\usepackage{amssymb}
\usepackage{amsthm}
\usepackage{xcolor}
\usepackage{microtype}
\usepackage{placeins}
\usepackage{needspace}
\usepackage{float}
\usepackage{fancyhdr}
\usepackage[ruled,vlined,linesnumbered]{algorithm2e}
\usepackage{wrapfig}
\usepackage{needspace}
\setcitestyle{authoryear,round,citesep={,},aysep={,},yysep={,}}


\newcommand{\method}{ICE}
\newcommand{\icename}{\textbf{I}nteraction-aware \textbf{C}lifford \textbf{E}ncoder}
\newcommand{\Cl}{\mathrm{Cl}}
\newcommand{\grade}[2]{\left\langle #1\right\rangle_{#2}}
\newcommand{\rev}[1]{\widetilde{#1}}

\newtheoremstyle{iceanalysis}{3pt}{3pt}{\normalfont}{}{\bfseries}{.}{0.5em}{}
\theoremstyle{iceanalysis}
\newtheorem{theorem}{Theorem}[section]
\newtheorem{definition}[theorem]{Definition}
\newtheorem{lemma}[theorem]{Lemma}
\title{ICE: Task-Aligned Clifford Latent Fields for Multimodal Graph Foundation Models}
\author{
 \begin{tabular}{c}
 Xunkai Li$^{1,*}$ \quad Xu Wang$^{2,*}$ \quad Yinlin Zhu$^{3}$ \quad Xiong Yongfu$^{1}$\\
 Yi Liu$^{1}$ \quad Rong-Hua Li$^{1}$ \quad Guoren Wang$^{1}$\\[6pt]
 $^{1}$Department of Computer Science, Beijing Institute of Technology, Beijing, China\\
 $^{2}$School of Airspace Science and Engineering, Shandong University, WeiHai, China\\
 $^{3}$School of Computer Science and Engineering, Sun Yat-sen University, GuangZhou, China
 \end{tabular}
}

\begin{document}
\maketitle
\footnotetext{*Equal contribution. Corresponding author: Rong-Hua Li \texttt{<lironghuabit@126.com>}.}

\begin{abstract}
Multimodal attributed graphs connect entities, visual content, language, and observed relations.
Learning one foundation across such graphs requires more than compressing each node into a fused
Euclidean vector. The representation must preserve entity semantics, construct interaction state
from graph neighborhoods, and expose that state to prediction units with different geometry. Our
empirical study shows why these requirements are inseparable. Higher-grade channels recover
pair relations across the foundation graphs, specialized queries reveal information hidden by
a generic readout, and rigid blade isolation removes cross-grade capacity. We therefore introduce
\method{} (\icename{}), a multimodal graph foundation model built on a
node-indexed Clifford latent field. Topology, text, and images enter explicit $\Cl(3)$ addresses.
Edge-aware geometric products transform these directions into scalar, bivector, and trivector
relations over observed neighborhoods. A protected Grade-1 route preserves entity semantics, while
the full grade and depth bank remains available to fresh node and link heads. We establish exact
cross-grade reachability, node-permutation equivariance, and a bound on the task residual around the
semantic score. Experiments span one shared foundation over eleven graphs, six node-classification
datasets, three link-prediction datasets, and matched few-shot tasks. ICE ranks first in all 30
reported supervised and few-shot comparisons. Core removals reduce every task summary, and
mechanism controls connect the gains to higher-order transport, retained multidepth structure,
semantic protection, and direct field access.
\end{abstract}

\section{Introduction}

Multimodal attributed graphs (MAGs) organize entities, textual and visual attributes, and observed
relations in one computable structure. A product can be ambiguous from its image and description
alone yet become identifiable through related products. A media item can be understood through its
synopsis and poster together with links induced by topic, audience, or creator. Graph learning must
therefore go beyond within-node modality matching. It must decide which cross-entity signals to
transport and which interactions matter to the current prediction
\citep{DBLP:journals/natmi/EktefaieDNFZ23,DBLP:journals/corr/abs-2402-05322,MAGB,mmgraph}.

Multimodal graph foundation models (MGFMs) raise this requirement from one dataset to many. Given
topology and text/image node attributes, one dataset-agnostic backbone should support both node
classification (NC) and link prediction (LP) across graphs. NC emphasizes entity-level class
signals, whereas LP scores whether two node states form a relation and is especially sensitive
to ordered interactions. A useful foundation must preserve entity-local semantics while exposing
relation-aware context through a reusable task interface.

Existing pipelines commonly encode each modality, fuse the outputs, and propagate one Euclidean
node vector. This design can learn strong representations, but it places topology, semantic
directions, and cross-modal interactions in an undifferentiated coordinate space. It does not
provide an algebraic operation that explicitly constructs higher-order relations from modality
directions along graph edges. Scalar attention can control a neighbor's contribution but does not retain a separate address for
the structured interaction being transported. A generic
downstream readout can then ignore relation state that the backbone already contains.

Task-specific Clifford graph models provide a useful starting point. LION represents multimodal
interactions with multivectors, so symmetric agreement and oriented relations remain accessible in
different algebraic grades \citep{LION}. A graph foundation introduces a broader requirement. One
encoder must construct those relations across heterogeneous graphs while preserving the entity
state required by node prediction and the pair state required by link prediction. The open problem
is therefore not whether Clifford products can model interaction, but how their structured state can
be learned once and reused by prediction units with different geometry.

Figure~\ref{fig:empirical} isolates three diagnostic views that a complete model
comparison cannot reveal. Panel (a) follows relation information as it expands across algebraic
grades and transport depth. Panel (b) shows that direct grade and depth access recovers node
information overlooked by a generic readout. Panel (c) shows that permanent blade isolation
removes useful cross-grade capacity. Together, these observations show that Clifford dimension
alone is not the answer. A successful foundation must coordinate relation construction,
representation freedom, and task access.

These findings motivate a continuous design across representation, transport, and task access. We propose
\method{} (\icename{}) as a dataset-agnostic multimodal graph foundation. ICE makes latent space
concrete as a node-indexed $\Cl(3)$ field whose grades describe scalar, directional, planar, and
volumetric coordinates. Task alignment is realized through fresh supervised and few-shot heads that
query this field by grade and depth.

ICE follows one continuous information path. It assigns topology and node modalities to explicit
input addresses, then composes their directions along observed graph neighborhoods. Product
magnitude controls neighbor allocation, while product direction transforms the transported value
and generates higher-grade relation state. An additive Grade-1 route keeps entity semantics
available. The complete H0, H1, and H2 field is exposed to fresh NC and LP heads, while semantic
topology contrast and masked geometric reconstruction train the same dataset-agnostic encoder.
This design preserves entity semantics without reducing graph interaction to a scalar weight or a
final pooled vector.

We evaluate ICE on supervised and few-shot node and link tasks spanning nine target graphs.
The results consistently favor ICE across saturated and difficult datasets. Component removals and
matched controls further connect the improvement to geometric transport, retained grade and depth
states, semantic protection, and task-specific access.

\noindent\textbf{Our Contributions.}
(1) \textit{\underline{Valuable Insights.}} We show that the limitation is an information-interface
problem: observed edges create relation state in multiple grades and depths, while a final Euclidean
readout can hide that state from a node or pair predictor. Controlled probes establish when those
states improve relation prediction and task-specific access.
(2) \textit{\underline{Novel Method.}} ICE is a dataset-agnostic Clifford field that lifts topology,
text, and image signals into explicit addresses, uses two edge-aware cross-grade transports,
retains the grade and depth bank, and protects a Grade-1 semantic route. Fresh NC and LP heads query the
same field at the prediction unit they must solve, and the formal results state the reachability,
equivariance, access, and residual guarantees under their assumptions.
(3) \textit{\underline{SOTA Performance.}} Under a shared cross-graph protocol, ICE ranks
first across six NC and three LP datasets and the matched few-shot tasks. The three task-organized
matrices report supervised and low-label results, while component ablations and mechanism controls
connect the gains to interaction construction, field retention, semantic protection, and task access.

\section{Preliminaries and Related Work}

\subsection{Problem Setup and Notation}

Consider a multimodal graph $\mathcal{G}_d=(\mathcal{V}_d,\mathcal{E}_d,\mathcal{R}_d)$ with
$N_d=|\mathcal{V}_d|$ nodes, observed edges $\mathcal{E}_d$, and raw text and image attributes
$\mathcal{R}_d$. Frozen modality encoders $\phi_m$ produce
$\mathbf{X}^{(m)}_d=\phi_m(\mathcal{R}^{(m)}_d)\in\mathbb{R}^{N_d\times d_m}$ for
$m\in\mathcal{M}=\{t,v\}$. The foundation encoder receives only
$(\mathcal{V}_d,\mathcal{E}_d,\{\mathbf{X}^{(m)}_d\}_{m\in\mathcal{M}})$; dataset identity,
target labels, validation decisions, and output dimensions are unavailable to its shared maps.

Foundation learning fits one encoder $F_\theta$ over the training graphs with self-supervised
objectives. At transfer time, $F_\theta$ is reused and a fresh task-local head is fitted. For node
classification, the prediction unit is $i\in\mathcal{V}_d$ and the output is
$\hat y_i\in\mathcal{Y}_d$, evaluated by Accuracy and Macro-F1. For link prediction, the prediction
unit is a candidate pair $(i,j)$ and the output is a score $s_{ij}$, evaluated by MRR. The objective
is therefore a shared state that supports both units, not one common label space or one head shared
across tasks.

This setting fixes the graph and frozen modality features presented to the encoder while allowing
the downstream output space and prediction unit to change. It differs from task-specific
multimodal graph learning, which may optimize one encoder for one graph and one task, and from
text-attributed graph foundations, which do not have to keep image--text interaction state
available. These boundaries specify what must be shared and what may remain task local.

Let $\mathcal{D}_{\mathrm{pre}}$ denote the graphs used for foundation learning and
$\mathcal{T}_d$ the downstream task on graph $d$. The shared parameters $\theta^\star$ are fitted
without target-task labels. A task-local parameter set $\omega_{d,\mathcal{T}_d}$ is then trained
from the designated training split and selected with validation data. Evaluation reads the
selected model once on held-out examples. The same $\theta^\star$ is used for all target graphs
and both task families, whereas $\omega_{d,\mathcal{T}_d}$ may change with the output space.
Thus the protocol tests reuse of one foundation, not reuse of a classifier trained elsewhere.

\subsection{Multimodal Graph Learning}

Multimodal graph learning combines entity attributes with relational structure. Traditional models
propagate visual and textual features on one target graph or learn modality-specific interaction for
one prediction setting \citep{MMGCN,MGAT,MAGB,mmgraph}. These models provide strong specialized
operators, but their representation boundary is chosen together with the target task. A node-level
readout can therefore discard pair information, while a relation-oriented readout can weaken the
entity semantics needed by node classification.

Geometric multimodal models offer a different interaction primitive. LION uses Clifford products
to separate scalar agreement from higher-grade oriented interaction \citep{LION}. This algebraic
structure is useful for representing multimodal relations, but a task-specific model does not by
itself define which intermediate grades and depths a cross-graph foundation must expose to a new
prediction unit.

\subsection{Graph Foundation Models}

Graph foundation models learn reusable structural or semantic representations through
self-supervision and lightweight adaptation \citep{graphmae2,SAMGPT,RiemannGFM,wang2024gft}.
Text-attributed models align graph structure with language, while multimodal foundations such as
UniGraph2 and PLANET extend this interface to visual attributes \citep{UniGraph2,PLANET}. Their
cross-graph training establishes the foundation setting, but the downstream interface is still
usually a final Euclidean node state.

For node classification, a fresh head maps one node state to a graph-specific label space. For
link prediction, a fresh head scores a pair against candidates. A final Euclidean node state can
serve either interface, but it need not preserve the relation information constructed at earlier
transport depths. Clifford products give one way to construct that information, while a reusable
foundation must also decide which states to retain for an unknown future prediction unit.
Section~\ref{sec:empirical} examines this access question before
Section~\ref{sec:methodology} introduces the Clifford field.

\section{Empirical Motivation}
\label{sec:empirical}

We first examine whether a multimodal graph foundation requires more than a wider Euclidean
embedding. The probes ask whether observed edges create relation state, how that state develops
across grades and transport depths, and whether retained grade and depth states support task-specific
access. They use 9,810 matched pairs, controlled blade isolation, and direct
grade and depth access (Appendix~\ref{app:probe-values}).

Panel (a) evaluates grade expansion and transport depth within one relation sequence. Panel (b)
compares direct grade and depth access with a generic readout. Panel (c) tests blade isolation while
keeping the basis channels. These probes separate interaction, transport, retention, and access
before the complete method is introduced.

% derived LaTeX asset for this paper.
\begingroup
\setlength{\intextsep}{6pt}
\begin{figure}[H]
\begingroup
\setlength{\abovecaptionskip}{2pt}
\centering
\includegraphics[width=\linewidth]{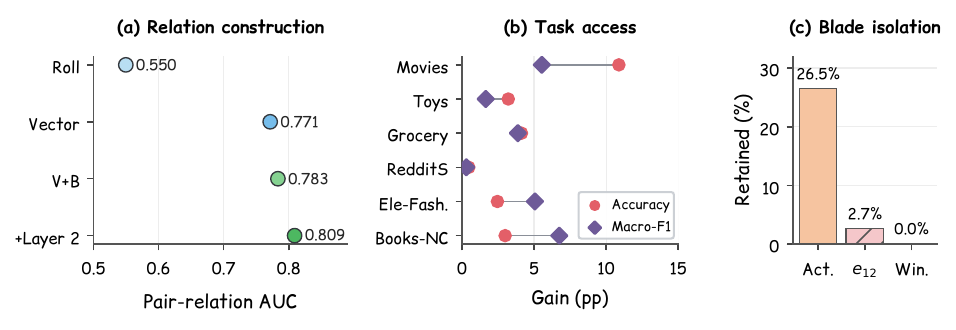}
\caption{\textbf{Empirical study.} (a) Pair-relation AUC after grade expansion and a second
transport layer. (b) Accuracy and Macro-F1 gains (percentage points) from direct grade--depth
access over a generic readout. (c) Rigid blade isolation reports the retained
active share (``Act.''), deepest $e_{12}$ energy (``$e_{12}$''), and matched validation wins
(``Win.'').}
\label{fig:empirical}
\endgroup
\end{figure}
\endgroup
 
Figure~\ref{fig:empirical}(a) separates relation construction from later composition.
Adding vector and bivector coordinates raises pair-relation AUC from 0.550 to 0.783; a second
transport raises it to 0.809. The first contrast is larger than the second, so the observation
supports cross-grade construction followed by complementary context, not depth as a substitute for
the geometric product. Figure~\ref{fig:empirical}(b) shows that direct grade-depth
access improves both NC metrics on all six examined
datasets. The gain ranges from 0.47 to 10.89 points for Accuracy and from 0.30 to 6.75 for Macro-F1.
Rigid blade isolation retains 26.53 percent of active parameters and 2.658 percent of deepest
$e_{12}$ energy, and wins none of the six matched comparisons (Fig.~\ref{fig:empirical}(c)).

The fixed-field Accuracy/Macro-F1 gains for Movies, Toys, Grocery, RedditS, Ele-Fashion, and
Books-NC are 10.89/5.54, 3.21/1.65, 4.13/3.88, 0.47/0.30, 2.46/5.06, and 3.00/6.75 points,
respectively. These fixed-field comparisons diagnose task access, not multi-seed transfer or LP
performance. Blade isolation preserves the coordinates but blocks learned cross-grade exchange; its
zero matched wins argue against permanent isolation, not against Clifford capacity in general. The
probes do not establish downstream transfer. They identify access and cross-grade-use requirements
that the full evaluation must test separately.

The task-access probe holds the encoded field fixed and changes only the readout. It therefore
tests whether retained states are usable by a new prediction unit, rather than whether the encoder
can be trained more effectively with a different objective. Conversely, blade isolation changes
the available interactions during transport. Reading the two controls together distinguishes
information that was never constructed from information constructed but hidden by a readout.

Together, these results suggest that graph transport creates prediction-relevant state
progressively. The first layer separates agreement from oriented interaction, and the second
composes that interaction with wider context. The probes jointly motivate preserving relation
state for pair prediction while keeping retained states directly accessible to downstream task
heads. Compressing these states before the task is known removes exactly the distinction that a
reusable foundation needs.

This observation leads to an encoder that constructs relations over observed edges, retains
intermediate grades and depths, and defers selection until a fresh head is fitted. Grade expansion
calls for Clifford relation construction, depth expansion for a second transport, and the access
probe for a retained grade-depth bank rather than a final-state readout. Blade isolation argues
for learned cross-grade exchange. Together with the need to preserve entity semantics for node
prediction, these findings motivate five design choices. ICE uses Clifford relation construction,
two-layer transport, a retained grade and depth bank, a protected Grade-1 semantic route, and
task-aligned field access. The controlled probes identify where relation information is
constructed and accessed. Section~\ref{sec:methodology} develops the five choices as one information
path, and Section~5 evaluates whether that path survives fresh heads, label
budgets, and matched component removals.

\begin{figure}[t]
\centering
\includegraphics[width=\linewidth]{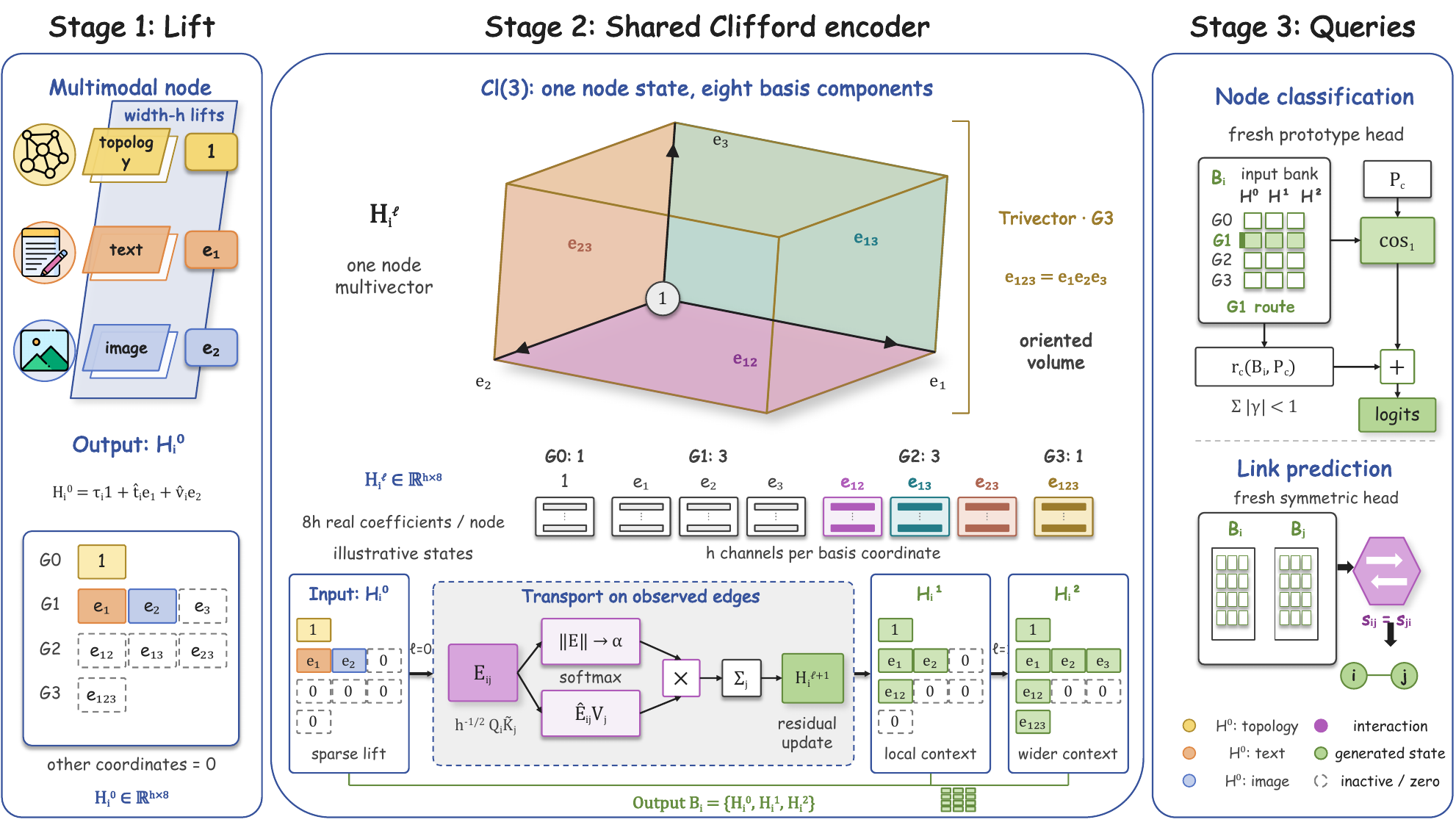}
\caption{\textbf{ICE framework.} One dataset-agnostic Clifford field connects multimodal input,
edge-aware graph transport, foundation learning, and fresh node and link heads.}
\label{fig:framework}
\end{figure}
\FloatBarrier

\section{Methodology}
\label{sec:methodology}

\subsection{Overview}

ICE keeps each graph-induced interaction addressable by algebraic grade and transport depth until a
task chooses its readout. The data flow is topology, text, and image lift, followed by two edge-aware
transports that produce and retain H0, H1, and H2 before fresh node or link heads query the field.
Read Fig.~\ref{fig:framework} from left to right: arrows are data flow, and the repeated field block is
the shared encoder reused across graphs.

\subsection{Clifford Field Construction and Learning}

Entity attributes describe a node, whereas graph neighborhoods describe how entities and modalities
interact. ICE uses $\Cl(3)$ because its geometric product represents symmetric agreement and an
oriented relation in one state, which later graph transport can compose.

\noindent\textbf{Multimodal and topological addresses.}
Lightweight adapters map text, image, and local topology at node $i$ to width-$h$ vectors
$\mathbf{t}_i$, $\mathbf{v}_i$, and $\boldsymbol{\tau}_i$. ICE initializes
\begin{equation}
\mathbf{H}^{0}_i=\boldsymbol{\tau}_i1+\widehat{\mathbf{t}}_i e_1+
\widehat{\mathbf{v}}_i e_2,
\label{eq:lift}
\end{equation}
where hats denote normalization. Topology enters the scalar address and the modalities enter
distinct vector addresses. Higher grades begin at zero and must therefore be produced by graph
computation. Learned maps may exchange information across grades. The addresses specify entry
points, not fixed destinations.

\noindent\textbf{Edge-wise relation construction.}
At layer $\ell$, shared multivector maps produce $\mathbf{Q}^{\ell}$,
$\mathbf{K}^{\ell}$, and $\mathbf{V}^{\ell}$. For every observed edge $j\!\to\!i$, ICE forms
$\mathbf{E}^{\ell}_{ij}=\mathbf{Q}^{\ell}_{i}\rev{\mathbf{K}^{\ell}_{j}}/\sqrt{h}$ and transports
the neighboring value as
\begin{align}
\alpha^{\ell}_{ij}
  &=\operatorname{softmax}_{j\in\mathcal{N}(i)}
  \left(\operatorname{mean}_{c}\left\|\mathbf{E}^{\ell}_{ij,c}\right\|/\tau\right), \\
\mathbf{U}^{\ell}_{i}
  &=\sum_{j\in\mathcal{N}(i)}\alpha^{\ell}_{ij}
  \left(\frac{\mathbf{E}^{\ell}_{ij}}
  {\left\|\mathbf{E}^{\ell}_{ij}\right\|+\epsilon}\right)\mathbf{V}^{\ell}_{j}.
\label{eq:transport}
\end{align}
Its magnitude controls neighbor mass, whereas its normalized multivector direction transforms the
value. A residual projection, normalization, and feed-forward map then produce
$\mathbf{H}^{\ell+1}_i$ without collapsing relation type into a scalar attention weight.

\noindent\textbf{Two-layer grade and depth construction.}
The first layer creates scalar agreement and bivector interaction in H1. The second composes those
relations with a new graph direction, returning information to Grade 1 or extending it to Grade 3
in H2. ICE retains H0, H1, and H2 so task queries can access each construction stage.

\noindent\textbf{Foundation objectives and semantic protection.}
ICE learns the field with two complementary objectives and one protected path. Semantic topology
contrast aligns a masked node modality with the visible semantic context supplied by its neighbors.
Masked geometric reconstruction requires the transported field to recover the masked vector and
the higher-grade coordinates created from it. Their joint objective is
\begin{equation}
\mathcal{L}_{\mathrm{pre}}=\lambda_{\mathrm{sem}}\mathcal{L}_{\mathrm{sem}}
+\lambda_{\mathrm{geo}}\sum_{g=1}^{3}\omega_g
\left\|\grade{D_{\psi}(\mathbf{H}^{2}_{i,\mathrm{vis}})}{g}
-\grade{\mathbf{H}^{0}_{i}}{g}\right\|_2^2.
\label{eq:pretrain}
\end{equation}
Semantic topology contrast constrains entity meaning, and geometric reconstruction constrains
cross-grade content. An additive Grade-1 path stays present through the transported field, giving
node prediction a stable semantic route while geometric coordinates refine its score. The shared
maps receive only graph topology and node attributes, never a dataset identifier or target label
space. Each transport output is consumed by the next layer or retained for a task head.

\subsection{Task-Aligned Field Access}

\begin{definition}[Addressable Clifford field]
For node $i$, ICE exports
$\mathcal{B}_i=\{\grade{\mathbf{H}^{\ell}_i}{g}\mid
\ell\in\{0,1,2\},\ g\in\{0,1,2,3\}\}$. H0 contains the multimodal addresses from
Eq.~\ref{eq:lift}; H1 and H2 contain the relations formed by Eq.~\ref{eq:transport}. This retained
bank is exactly the shared-encoder output read by every fresh task head.
\label{def:addressable-field}
\end{definition}

\noindent\textbf{Node prediction.}
Node classification reads $\mathcal{B}_i$. A prototype $\mathbf{P}_c$ supplies an always-present
Grade-1 cosine score, while bounded coefficients query the remaining addresses.
\begin{equation}
s^{\mathrm{NC}}_{ic}=\cos_1(\mathbf{H}^{\mathrm{pool}}_i,\mathbf{P}_c)
+\sum_{\ell,g}\gamma_{c\ell g}\cos_g(\mathbf{H}^{\ell}_i,\mathbf{P}_c),
\qquad \sum_{\ell,g}|\gamma_{c\ell g}|\leq\rho<1.
\label{eq:nc-head}
\end{equation}
The shared field remains unchanged when the head converts these scores into predictions.

\noindent\textbf{Pair prediction.}
Link prediction reads two fields. Grade-compatible projections first assemble
$\mathbf{z}_i=\mathop{\Vert}_{\ell,g}q_{\ell g}(\grade{\mathbf{H}^{\ell}_i}{g})$. The pair head
then uses
\begin{equation}
\mathbf{p}_{ij}=\left[\mathbf{z}_i+\mathbf{z}_j\,\middle\Vert\,
|\mathbf{z}_i-\mathbf{z}_j|\,\middle\Vert\,
\mathbf{z}_i\odot\mathbf{z}_j\right],
\qquad s^{\mathrm{LP}}_{ij}=h_{\mathrm{LP}}(\mathbf{p}_{ij}).
\label{eq:lp-head}
\end{equation}
The symmetric descriptor combines shared content, disagreement, and multiplicative agreement.
Node and link heads differ in prediction unit, not in the dataset-agnostic foundation they receive.

\subsection{Theoretical Analysis}

The analysis separates standard Clifford-algebra identities from properties of ICE's storage,
transport, and heads. We use the direct-sum grade decomposition and geometric-product identities
\citep{dorst2007geometric,LION}. Permutation equivariance follows the shared message-passing
construction used in graph neural networks \citep{GraphSAGE,gin}. Appendix
Section~\ref{app:properties} proves the ICE-specific statements under their explicit assumptions.

\begin{theorem}[Cross-grade reachability]
\label{thm:cross-grade}
For vectors $a,b,c\in\mathbb{R}^{3}\subset\Cl(3)$, the geometric products satisfy
\begin{equation}
ab=a\cdot b+a\wedge b,
\qquad
(a\wedge b)c=\grade{(a\wedge b)c}{1}+\grade{(a\wedge b)c}{3}.
\label{eq:main-grade-routing}
\end{equation}
Thus vector interaction contains scalar agreement and a bivector, while composing the bivector with
a vector lies in Grades 1 and 3. In ICE these are the local interaction formed by the first
edge-wise transport and the broader composition available to the second transport. This establishes
which grades are reachable for a GFM reuse interface. It does not imply that every grade is nonzero
or predictive for every graph.
\end{theorem}

\begin{lemma}[Complete grade and depth access]
\label{lem:field-access}
If ICE retains every grade at depths H0, H1, and H2, then each transported state is uniquely
recoverable:
\begin{equation}
\mathbf{H}^{\ell}_i=\sum_{g=0}^{3}\grade{\mathbf{H}^{\ell}_i}{g},
\qquad \ell\in\{0,1,2\}.
\label{eq:main-field-completeness}
\end{equation}
This is the complete-access meaning of Definition~\ref{def:addressable-field}. The retained field
is the output of the shared encoder; an NC head reads node states and an LP head reads a pair, so
both prediction units can select from the same dataset-agnostic state without requiring the encoder
to collapse grades or depths in advance.
\end{lemma}

\begin{theorem}[Node-permutation equivariance]
\label{thm:equivariance}
Assume shared nodewise maps, permutation-equivariant neighborhood aggregation over the observed
edges, and the pair descriptor of Eq.~\ref{eq:lp-head}. For every node permutation matrix
$\mathbf{P}$,
\begin{equation}
F_{\theta}(\mathbf{P}\mathcal{G})=\mathbf{P}F_{\theta}(\mathcal{G}).
\label{eq:main-equivariance}
\end{equation}
Shared ICE transport therefore reindexes messages when node labels change and remains a valid graph
operator. The NC output follows the same relabeling, while the symmetric LP descriptor preserves
the score for an exchanged pair. The assumption matches the common transport map across graphs and
does not assert invariance to changes in graph structure.
\end{theorem}

\begin{lemma}[Protected Grade-1 Semantic Route]
\label{lem:semantic-residual}
If the NC head in Eq.~\ref{eq:nc-head} obeys
$\sum_{\ell,g}|\gamma_{c\ell g}|\leq\rho<1$ and each grade-wise cosine has magnitude at most one,
then
\begin{equation}
\left|s^{\mathrm{NC}}_{ic}-\cos_1(\mathbf{H}^{\mathrm{pool}}_i,\mathbf{P}_c)\right|\leq\rho.
\label{eq:main-semantic-bound}
\end{equation}
The Grade-1 prototype score in ICE is the protected node route. The bound limits how far task
residuals over other grades and depths can move that score, so a fresh node head can add interaction
information without replacing the base semantic path.
\end{lemma}

For a sampled subgraph, ICE requires
$O(L|E_b|hB^2+L|V_b|hB)$ transport time and
$O(L|V_b|hB+|E_b|hB)$ retained state. Since $B=8$ and $L=2$, the implementation is linear in the
sampled graph size up to fixed Clifford factors; Appendix~\ref{app:properties} gives the derivation.

\section{Experiments}

We evaluate whether one ICE foundation transfers across prediction units and label regimes, then ask
whether the observed behavior follows the proposed construction--retention--access path.
Appendix~\ref{app:protocol} gives the complete data, optimization, and evaluation settings.

\noindent\textbf{Experimental settings.}
One encoder is pretrained for five epochs over eleven multimodal graphs. Six node-classification
and three link-prediction datasets use fresh ten-epoch heads. We report NC Accuracy and Macro-F1,
LP MRR against 150 candidates, and 3/5/10-shot Accuracy for 5-way NC and balanced 2-way link
classification. Every entry is the mean and sample standard deviation over eight validation-selected
runs. Comparators span supervised, multimodal, self-supervised, graph-foundation, and multimodal
graph-foundation models
\citep{MMGCN,MGAT,graphmae2,SAMGPT,RiemannGFM,UniGraph2,PLANET}.

\subsection{Transfer Across Prediction Units and Label Budgets}

We first ask whether one pretrained field can serve both node and pair prediction when a fresh
head receives either full supervision or only a few labels. Tables~\ref{tab:main-nc}--\ref{tab:main-few-nc}
report supervised NC, supervised LP, few-shot link classification, and few-shot NC.
Shared-foundation methods remain separate from task-specific specialists, whose training opportunity
differs (Appendix~\ref{app:protocol}).

% derived LaTeX asset for this paper.
\begin{table}[t]
\centering
\caption{\textbf{Node classification.} Accuracy and Macro-F1 for each dataset over eight seeds.
Best means are bold and second-best means are underlined.}
\label{tab:main-results}
\label{tab:main-nc}
\begingroup
\fontsize{6.4}{7.0}\selectfont
\setlength{\tabcolsep}{0pt}
\renewcommand{\arraystretch}{1.04}
\begin{tabular*}{\linewidth}{@{\extracolsep{\fill}}l*{12}{c}@{}}
\toprule
\textbf{Method} & \multicolumn{2}{c}{\textbf{RedditS}} & \multicolumn{2}{c}{\textbf{Movies}} &
\multicolumn{2}{c}{\textbf{Grocery}} & \multicolumn{2}{c}{\textbf{Toys}} &
\multicolumn{2}{c}{\textbf{Ele-F.}} & \multicolumn{2}{c}{\textbf{Books-NC}} \\
\cmidrule(lr){2-3}\cmidrule(lr){4-5}\cmidrule(lr){6-7}\cmidrule(lr){8-9}\cmidrule(lr){10-11}\cmidrule(l){12-13}
 & \textbf{Acc.} & \textbf{F1} & \textbf{Acc.} & \textbf{F1} & \textbf{Acc.} & \textbf{F1} &
\textbf{Acc.} & \textbf{F1} & \textbf{Acc.} & \textbf{F1} & \textbf{Acc.} & \textbf{F1} \\
\midrule
\multicolumn{13}{l}{\textit{Traditional multimodal}} \\
MMGCN & 90.27{\fontsize{4.85}{5.10}\selectfont $\pm$0.34} & 84.22{\fontsize{4.85}{5.10}\selectfont $\pm$0.73} & 53.41{\fontsize{4.85}{5.10}\selectfont $\pm$1.03} & 41.66{\fontsize{4.85}{5.10}\selectfont $\pm$2.03} & 82.56{\fontsize{4.85}{5.10}\selectfont $\pm$0.49} & 73.83{\fontsize{4.85}{5.10}\selectfont $\pm$0.93} & 80.02{\fontsize{4.85}{5.10}\selectfont $\pm$0.64} & 76.36{\fontsize{4.85}{5.10}\selectfont $\pm$1.23} & 86.59{\fontsize{4.85}{5.10}\selectfont $\pm$0.08} & 68.85{\fontsize{4.85}{5.10}\selectfont $\pm$0.35} & 83.22{\fontsize{4.85}{5.10}\selectfont $\pm$0.10} & 71.28{\fontsize{4.85}{5.10}\selectfont $\pm$0.22} \\
MGAT & 92.78{\fontsize{4.85}{5.10}\selectfont $\pm$0.50} & 87.27{\fontsize{4.85}{5.10}\selectfont $\pm$0.53} & 53.87{\fontsize{4.85}{5.10}\selectfont $\pm$0.50} & 44.09{\fontsize{4.85}{5.10}\selectfont $\pm$1.60} & 83.74{\fontsize{4.85}{5.10}\selectfont $\pm$0.62} & 74.77{\fontsize{4.85}{5.10}\selectfont $\pm$1.11} & 79.61{\fontsize{4.85}{5.10}\selectfont $\pm$0.74} & 77.09{\fontsize{4.85}{5.10}\selectfont $\pm$0.87} & 84.84{\fontsize{4.85}{5.10}\selectfont $\pm$0.08} & 69.62{\fontsize{4.85}{5.10}\selectfont $\pm$0.21} & 82.91{\fontsize{4.85}{5.10}\selectfont $\pm$0.04} & 71.45{\fontsize{4.85}{5.10}\selectfont $\pm$0.11} \\
\midrule
\multicolumn{13}{l}{\textit{Self-supervised graph models}} \\
GRACE & 93.01{\fontsize{4.85}{5.10}\selectfont $\pm$0.53} & 88.39{\fontsize{4.85}{5.10}\selectfont $\pm$1.12} & 48.09{\fontsize{4.85}{5.10}\selectfont $\pm$0.97} & 37.18{\fontsize{4.85}{5.10}\selectfont $\pm$1.33} & 70.83{\fontsize{4.85}{5.10}\selectfont $\pm$0.81} & 60.69{\fontsize{4.85}{5.10}\selectfont $\pm$1.05} & 72.82{\fontsize{4.85}{5.10}\selectfont $\pm$0.66} & 69.09{\fontsize{4.85}{5.10}\selectfont $\pm$0.63} & 83.58{\fontsize{4.85}{5.10}\selectfont $\pm$0.11} & 70.09{\fontsize{4.85}{5.10}\selectfont $\pm$0.47} & 74.96{\fontsize{4.85}{5.10}\selectfont $\pm$0.06} & 70.09{\fontsize{4.85}{5.10}\selectfont $\pm$0.11} \\
GraphMAE2 & 92.81{\fontsize{4.85}{5.10}\selectfont $\pm$0.44} & 87.93{\fontsize{4.85}{5.10}\selectfont $\pm$0.37} & 50.08{\fontsize{4.85}{5.10}\selectfont $\pm$0.77} & 38.68{\fontsize{4.85}{5.10}\selectfont $\pm$1.63} & 76.24{\fontsize{4.85}{5.10}\selectfont $\pm$0.60} & 66.74{\fontsize{4.85}{5.10}\selectfont $\pm$1.33} & 75.11{\fontsize{4.85}{5.10}\selectfont $\pm$0.52} & 71.80{\fontsize{4.85}{5.10}\selectfont $\pm$0.50} & 83.32{\fontsize{4.85}{5.10}\selectfont $\pm$0.31} & 65.92{\fontsize{4.85}{5.10}\selectfont $\pm$0.59} & 74.15{\fontsize{4.85}{5.10}\selectfont $\pm$0.22} & 69.20{\fontsize{4.85}{5.10}\selectfont $\pm$0.28} \\
\midrule
\multicolumn{13}{l}{\textit{Graph foundation models}} \\
RiemannGFM & 91.63{\fontsize{4.85}{5.10}\selectfont $\pm$0.45} & 85.20{\fontsize{4.85}{5.10}\selectfont $\pm$1.13} & 52.80{\fontsize{4.85}{5.10}\selectfont $\pm$0.43} & 40.74{\fontsize{4.85}{5.10}\selectfont $\pm$1.25} & 82.62{\fontsize{4.85}{5.10}\selectfont $\pm$0.46} & 74.90{\fontsize{4.85}{5.10}\selectfont $\pm$1.55} & 77.85{\fontsize{4.85}{5.10}\selectfont $\pm$0.45} & 74.84{\fontsize{4.85}{5.10}\selectfont $\pm$0.60} & 87.07{\fontsize{4.85}{5.10}\selectfont $\pm$0.20} & 70.45{\fontsize{4.85}{5.10}\selectfont $\pm$1.32} & 78.13{\fontsize{4.85}{5.10}\selectfont $\pm$0.17} & 70.73{\fontsize{4.85}{5.10}\selectfont $\pm$0.24} \\
GFT & 93.02{\fontsize{4.85}{5.10}\selectfont $\pm$0.37} & 87.00{\fontsize{4.85}{5.10}\selectfont $\pm$2.03} & 51.33{\fontsize{4.85}{5.10}\selectfont $\pm$0.67} & 28.14{\fontsize{4.85}{5.10}\selectfont $\pm$1.82} & 76.80{\fontsize{4.85}{5.10}\selectfont $\pm$2.22} & 59.11{\fontsize{4.85}{5.10}\selectfont $\pm$3.02} & 79.52{\fontsize{4.85}{5.10}\selectfont $\pm$0.58} & 76.00{\fontsize{4.85}{5.10}\selectfont $\pm$0.92} & 87.14{\fontsize{4.85}{5.10}\selectfont $\pm$0.22} & 70.33{\fontsize{4.85}{5.10}\selectfont $\pm$1.24} & 75.93{\fontsize{4.85}{5.10}\selectfont $\pm$0.41} & 66.18{\fontsize{4.85}{5.10}\selectfont $\pm$0.30} \\
SAMGPT & 93.11{\fontsize{4.85}{5.10}\selectfont $\pm$0.19} & 87.12{\fontsize{4.85}{5.10}\selectfont $\pm$0.64} & 50.25{\fontsize{4.85}{5.10}\selectfont $\pm$0.28} & 34.21{\fontsize{4.85}{5.10}\selectfont $\pm$1.37} & 76.41{\fontsize{4.85}{5.10}\selectfont $\pm$0.54} & 63.40{\fontsize{4.85}{5.10}\selectfont $\pm$0.88} & 73.81{\fontsize{4.85}{5.10}\selectfont $\pm$0.41} & 67.12{\fontsize{4.85}{5.10}\selectfont $\pm$0.73} & 83.81{\fontsize{4.85}{5.10}\selectfont $\pm$0.11} & 69.73{\fontsize{4.85}{5.10}\selectfont $\pm$0.39} & 74.29{\fontsize{4.85}{5.10}\selectfont $\pm$0.11} & 66.57{\fontsize{4.85}{5.10}\selectfont $\pm$0.29} \\
\midrule
\multicolumn{13}{l}{\textit{Multimodal graph foundations}} \\
UniGraph2 & 93.65{\fontsize{4.85}{5.10}\selectfont $\pm$0.17} & 87.91{\fontsize{4.85}{5.10}\selectfont $\pm$0.68} & 53.02{\fontsize{4.85}{5.10}\selectfont $\pm$0.53} & 43.43{\fontsize{4.85}{5.10}\selectfont $\pm$1.86} & 82.10{\fontsize{4.85}{5.10}\selectfont $\pm$0.37} & 73.93{\fontsize{4.85}{5.10}\selectfont $\pm$1.37} & 79.00{\fontsize{4.85}{5.10}\selectfont $\pm$0.59} & 76.02{\fontsize{4.85}{5.10}\selectfont $\pm$0.78} & 87.06{\fontsize{4.85}{5.10}\selectfont $\pm$0.18} & 69.80{\fontsize{4.85}{5.10}\selectfont $\pm$0.82} & 79.06{\fontsize{4.85}{5.10}\selectfont $\pm$0.27} & 68.74{\fontsize{4.85}{5.10}\selectfont $\pm$0.36} \\
PLANET & \underline{96.62}{\fontsize{4.85}{5.10}\selectfont $\pm$0.22} & \underline{92.44}{\fontsize{4.85}{5.10}\selectfont $\pm$0.43} & \underline{57.06}{\fontsize{4.85}{5.10}\selectfont $\pm$0.61} & \underline{47.49}{\fontsize{4.85}{5.10}\selectfont $\pm$1.23} & \underline{85.16}{\fontsize{4.85}{5.10}\selectfont $\pm$0.88} & \underline{77.23}{\fontsize{4.85}{5.10}\selectfont $\pm$0.86} & \underline{81.22}{\fontsize{4.85}{5.10}\selectfont $\pm$0.50} & \underline{77.55}{\fontsize{4.85}{5.10}\selectfont $\pm$0.75} & \underline{87.37}{\fontsize{4.85}{5.10}\selectfont $\pm$0.12} & \underline{70.74}{\fontsize{4.85}{5.10}\selectfont $\pm$0.78} & \underline{84.16}{\fontsize{4.85}{5.10}\selectfont $\pm$0.07} & \underline{74.43}{\fontsize{4.85}{5.10}\selectfont $\pm$0.18} \\
\midrule
\textbf{ICE} & \textbf{96.97}{\fontsize{4.85}{5.10}\selectfont $\pm$0.16} & \textbf{92.70}{\fontsize{4.85}{5.10}\selectfont $\pm$0.14} & \textbf{59.42}{\fontsize{4.85}{5.10}\selectfont $\pm$0.48} & \textbf{50.36}{\fontsize{4.85}{5.10}\selectfont $\pm$0.82} & \textbf{86.72}{\fontsize{4.85}{5.10}\selectfont $\pm$0.39} & \textbf{79.24}{\fontsize{4.85}{5.10}\selectfont $\pm$0.52} & \textbf{82.64}{\fontsize{4.85}{5.10}\selectfont $\pm$0.33} & \textbf{79.06}{\fontsize{4.85}{5.10}\selectfont $\pm$0.45} & \textbf{88.21}{\fontsize{4.85}{5.10}\selectfont $\pm$0.10} & \textbf{72.08}{\fontsize{4.85}{5.10}\selectfont $\pm$0.54} & \textbf{85.24}{\fontsize{4.85}{5.10}\selectfont $\pm$0.08} & \textbf{75.96}{\fontsize{4.85}{5.10}\selectfont $\pm$0.22} \\
\bottomrule
\end{tabular*}
\endgroup
\end{table}

\begin{table}[t]
\centering
\caption{\textbf{Link prediction and few-shot link classification.} MRR for supervised link
prediction and Accuracy for balanced 2-way classification over eight seeds.}
\label{tab:main-lp}
\begingroup
\fontsize{7.8}{8.6}\selectfont
\setlength{\tabcolsep}{2.0pt}
\renewcommand{\arraystretch}{1.08}
\begin{tabular*}{\linewidth}{@{\extracolsep{\fill}}l*{9}{c}@{}}
\toprule
\textbf{Method} & \multicolumn{3}{c}{\textbf{Link prediction}} & \multicolumn{6}{c}{\textbf{Few-shot link classification}} \\
\cmidrule(lr){2-4}\cmidrule(l){5-10}
 & \textbf{Sports} & \textbf{Cloth} & \textbf{Books-LP} & \textbf{Sports 3} & \textbf{Sports 5} & \textbf{Sports 10} & \textbf{Cloth 3} & \textbf{Cloth 5} & \textbf{Cloth 10} \\
\midrule
\multicolumn{10}{l}{\textit{Traditional multimodal}} \\
MMGCN & 23.44{\fontsize{4.85}{5.10}\selectfont $\pm$0.43} & 17.74{\fontsize{4.85}{5.10}\selectfont $\pm$0.38} & 20.73{\fontsize{4.85}{5.10}\selectfont $\pm$0.48} & 55.86{\fontsize{4.85}{5.10}\selectfont $\pm$2.96} & 53.70{\fontsize{4.85}{5.10}\selectfont $\pm$1.20} & 56.66{\fontsize{4.85}{5.10}\selectfont $\pm$1.60} & 64.36{\fontsize{4.85}{5.10}\selectfont $\pm$4.07} & 68.61{\fontsize{4.85}{5.10}\selectfont $\pm$5.38} & 67.27{\fontsize{4.85}{5.10}\selectfont $\pm$4.20} \\
MGAT & 21.74{\fontsize{4.85}{5.10}\selectfont $\pm$0.96} & 15.47{\fontsize{4.85}{5.10}\selectfont $\pm$0.32} & 21.82{\fontsize{4.85}{5.10}\selectfont $\pm$0.53} & 54.92{\fontsize{4.85}{5.10}\selectfont $\pm$2.40} & 56.55{\fontsize{4.85}{5.10}\selectfont $\pm$2.69} & 57.92{\fontsize{4.85}{5.10}\selectfont $\pm$1.80} & 67.05{\fontsize{4.85}{5.10}\selectfont $\pm$1.45} & 70.34{\fontsize{4.85}{5.10}\selectfont $\pm$3.36} & 68.66{\fontsize{4.85}{5.10}\selectfont $\pm$2.94} \\
\midrule
\multicolumn{10}{l}{\textit{Self-supervised graph models}} \\
GRACE & 25.31{\fontsize{4.85}{5.10}\selectfont $\pm$0.16} & 18.27{\fontsize{4.85}{5.10}\selectfont $\pm$0.15} & 19.30{\fontsize{4.85}{5.10}\selectfont $\pm$0.27} & 56.42{\fontsize{4.85}{5.10}\selectfont $\pm$1.36} & 57.94{\fontsize{4.85}{5.10}\selectfont $\pm$2.58} & 58.50{\fontsize{4.85}{5.10}\selectfont $\pm$1.43} & 62.36{\fontsize{4.85}{5.10}\selectfont $\pm$2.75} & 64.94{\fontsize{4.85}{5.10}\selectfont $\pm$1.09} & 64.96{\fontsize{4.85}{5.10}\selectfont $\pm$1.86} \\
GraphMAE2 & 24.54{\fontsize{4.85}{5.10}\selectfont $\pm$0.30} & 18.69{\fontsize{4.85}{5.10}\selectfont $\pm$0.21} & 19.99{\fontsize{4.85}{5.10}\selectfont $\pm$0.19} & 53.28{\fontsize{4.85}{5.10}\selectfont $\pm$3.09} & 54.36{\fontsize{4.85}{5.10}\selectfont $\pm$1.65} & 56.05{\fontsize{4.85}{5.10}\selectfont $\pm$1.70} & 60.95{\fontsize{4.85}{5.10}\selectfont $\pm$1.61} & 61.39{\fontsize{4.85}{5.10}\selectfont $\pm$1.71} & 62.33{\fontsize{4.85}{5.10}\selectfont $\pm$1.78} \\
\midrule
\multicolumn{10}{l}{\textit{Graph foundation models}} \\
RiemannGFM & 21.92{\fontsize{4.85}{5.10}\selectfont $\pm$0.51} & 19.20{\fontsize{4.85}{5.10}\selectfont $\pm$0.44} & 22.03{\fontsize{4.85}{5.10}\selectfont $\pm$0.67} & 54.07{\fontsize{4.85}{5.10}\selectfont $\pm$1.46} & 53.73{\fontsize{4.85}{5.10}\selectfont $\pm$2.03} & 53.68{\fontsize{4.85}{5.10}\selectfont $\pm$1.24} & 63.45{\fontsize{4.85}{5.10}\selectfont $\pm$4.07} & 62.66{\fontsize{4.85}{5.10}\selectfont $\pm$1.75} & 66.20{\fontsize{4.85}{5.10}\selectfont $\pm$5.94} \\
GFT & 22.04{\fontsize{4.85}{5.10}\selectfont $\pm$0.62} & 17.63{\fontsize{4.85}{5.10}\selectfont $\pm$0.59} & 20.16{\fontsize{4.85}{5.10}\selectfont $\pm$1.21} & 56.75{\fontsize{4.85}{5.10}\selectfont $\pm$2.63} & 56.86{\fontsize{4.85}{5.10}\selectfont $\pm$2.30} & 55.73{\fontsize{4.85}{5.10}\selectfont $\pm$3.36} & 63.70{\fontsize{4.85}{5.10}\selectfont $\pm$2.84} & 66.66{\fontsize{4.85}{5.10}\selectfont $\pm$2.10} & 65.37{\fontsize{4.85}{5.10}\selectfont $\pm$4.03} \\
SAMGPT & 24.09{\fontsize{4.85}{5.10}\selectfont $\pm$0.22} & 16.41{\fontsize{4.85}{5.10}\selectfont $\pm$0.20} & 24.91{\fontsize{4.85}{5.10}\selectfont $\pm$0.38} & 59.41{\fontsize{4.85}{5.10}\selectfont $\pm$3.62} & 59.00{\fontsize{4.85}{5.10}\selectfont $\pm$3.56} & 60.48{\fontsize{4.85}{5.10}\selectfont $\pm$4.25} & 72.61{\fontsize{4.85}{5.10}\selectfont $\pm$1.90} & 74.20{\fontsize{4.85}{5.10}\selectfont $\pm$1.60} & 74.25{\fontsize{4.85}{5.10}\selectfont $\pm$1.75} \\
\midrule
\multicolumn{10}{l}{\textit{Multimodal graph foundations}} \\
UniGraph2 & 27.09{\fontsize{4.85}{5.10}\selectfont $\pm$0.13} & 19.31{\fontsize{4.85}{5.10}\selectfont $\pm$0.35} & 19.44{\fontsize{4.85}{5.10}\selectfont $\pm$0.19} & 60.83{\fontsize{4.85}{5.10}\selectfont $\pm$3.90} & 64.27{\fontsize{4.85}{5.10}\selectfont $\pm$2.74} & 65.08{\fontsize{4.85}{5.10}\selectfont $\pm$3.17} & 73.44{\fontsize{4.85}{5.10}\selectfont $\pm$1.57} & 71.28{\fontsize{4.85}{5.10}\selectfont $\pm$4.20} & 73.77{\fontsize{4.85}{5.10}\selectfont $\pm$1.91} \\
PLANET & \underline{27.51}{\fontsize{4.85}{5.10}\selectfont $\pm$0.14} & \underline{20.25}{\fontsize{4.85}{5.10}\selectfont $\pm$0.27} & \underline{27.62}{\fontsize{4.85}{5.10}\selectfont $\pm$0.25} & \underline{62.89}{\fontsize{4.85}{5.10}\selectfont $\pm$3.99} & \underline{64.36}{\fontsize{4.85}{5.10}\selectfont $\pm$3.05} & \underline{67.84}{\fontsize{4.85}{5.10}\selectfont $\pm$1.38} & \underline{74.03}{\fontsize{4.85}{5.10}\selectfont $\pm$4.12} & \underline{75.22}{\fontsize{4.85}{5.10}\selectfont $\pm$1.20} & \underline{75.44}{\fontsize{4.85}{5.10}\selectfont $\pm$2.84} \\
\midrule
\textbf{ICE} & \textbf{28.38}{\fontsize{4.85}{5.10}\selectfont $\pm$0.11} & \textbf{21.08}{\fontsize{4.85}{5.10}\selectfont $\pm$0.20} & \textbf{29.26}{\fontsize{4.85}{5.10}\selectfont $\pm$0.23} & \textbf{64.11}{\fontsize{4.85}{5.10}\selectfont $\pm$3.71} & \textbf{65.62}{\fontsize{4.85}{5.10}\selectfont $\pm$2.81} & \textbf{69.18}{\fontsize{4.85}{5.10}\selectfont $\pm$1.26} & \textbf{75.36}{\fontsize{4.85}{5.10}\selectfont $\pm$3.84} & \textbf{76.41}{\fontsize{4.85}{5.10}\selectfont $\pm$1.11} & \textbf{76.71}{\fontsize{4.85}{5.10}\selectfont $\pm$2.62} \\
\bottomrule
\end{tabular*}
\endgroup
\end{table}

\begin{table}[t]
\centering
\caption{\textbf{Few-shot node classification.} Focused 5-way comparison across 3, 5, and 10 shots; the complete matrix appears in Appendix Table~\ref{tab:few-shot-nc}.}
\label{tab:main-few-nc}
\begingroup
\fontsize{7.8}{8.6}\selectfont
\setlength{\tabcolsep}{2.0pt}
\renewcommand{\arraystretch}{1.08}
\begin{tabular*}{\linewidth}{@{\extracolsep{\fill}}l*{9}{c}@{}}
\toprule
\textbf{Method} & \multicolumn{3}{c}{\textbf{Grocery 5-way}} & \multicolumn{3}{c}{\textbf{Ele-Fashion 5-way}} & \multicolumn{3}{c}{\textbf{Books-NC 5-way}} \\
\cmidrule(lr){2-4}\cmidrule(lr){5-7}\cmidrule(l){8-10}
 & \textbf{3} & \textbf{5} & \textbf{10} & \textbf{3} & \textbf{5} & \textbf{10} & \textbf{3} & \textbf{5} & \textbf{10} \\
\midrule
GFT & 61.45{\fontsize{4.85}{5.10}\selectfont $\pm$2.22} & 64.60{\fontsize{4.85}{5.10}\selectfont $\pm$3.54} & 63.12{\fontsize{4.85}{5.10}\selectfont $\pm$4.07} & 59.08{\fontsize{4.85}{5.10}\selectfont $\pm$3.95} & 61.38{\fontsize{4.85}{5.10}\selectfont $\pm$3.41} & 62.28{\fontsize{4.85}{5.10}\selectfont $\pm$2.62} & 48.70{\fontsize{4.85}{5.10}\selectfont $\pm$4.96} & 50.95{\fontsize{4.85}{5.10}\selectfont $\pm$4.35} & 51.62{\fontsize{4.85}{5.10}\selectfont $\pm$4.73} \\
UniGraph2 & 60.05{\fontsize{4.85}{5.10}\selectfont $\pm$4.13} & 61.25{\fontsize{4.85}{5.10}\selectfont $\pm$3.09} & 66.27{\fontsize{4.85}{5.10}\selectfont $\pm$3.11} & 53.98{\fontsize{4.85}{5.10}\selectfont $\pm$4.07} & 58.73{\fontsize{4.85}{5.10}\selectfont $\pm$2.95} & 60.38{\fontsize{4.85}{5.10}\selectfont $\pm$2.08} & 59.67{\fontsize{4.85}{5.10}\selectfont $\pm$4.23} & 61.55{\fontsize{4.85}{5.10}\selectfont $\pm$3.97} & 63.80{\fontsize{4.85}{5.10}\selectfont $\pm$4.02} \\
PLANET & \underline{77.85}{\fontsize{4.85}{5.10}\selectfont $\pm$4.06} & \underline{79.88}{\fontsize{4.85}{5.10}\selectfont $\pm$3.27} & \underline{81.93}{\fontsize{4.85}{5.10}\selectfont $\pm$3.48} & \underline{70.50}{\fontsize{4.85}{5.10}\selectfont $\pm$4.37} & \underline{72.97}{\fontsize{4.85}{5.10}\selectfont $\pm$3.55} & \underline{74.85}{\fontsize{4.85}{5.10}\selectfont $\pm$3.73} & \underline{63.62}{\fontsize{4.85}{5.10}\selectfont $\pm$4.11} & \underline{67.59}{\fontsize{4.85}{5.10}\selectfont $\pm$4.68} & \underline{69.18}{\fontsize{4.85}{5.10}\selectfont $\pm$3.90} \\
\midrule
\textbf{ICE} & \textbf{78.96}{\fontsize{4.85}{5.10}\selectfont $\pm$3.78} & \textbf{81.02}{\fontsize{4.85}{5.10}\selectfont $\pm$3.08} & \textbf{83.10}{\fontsize{4.85}{5.10}\selectfont $\pm$3.21} & \textbf{71.61}{\fontsize{4.85}{5.10}\selectfont $\pm$4.01} & \textbf{74.07}{\fontsize{4.85}{5.10}\selectfont $\pm$3.33} & \textbf{76.02}{\fontsize{4.85}{5.10}\selectfont $\pm$3.42} & \textbf{64.83}{\fontsize{4.85}{5.10}\selectfont $\pm$3.86} & \textbf{68.78}{\fontsize{4.85}{5.10}\selectfont $\pm$4.32} & \textbf{70.42}{\fontsize{4.85}{5.10}\selectfont $\pm$3.61} \\
\bottomrule
\end{tabular*}
\endgroup
\end{table}

The supervised comparisons in Tables~\ref{tab:main-nc} and~\ref{tab:main-lp} place ICE first on
both NC metrics and LP MRR. The NC advantage over the next method spans 0.35--2.36 points, while
the three LP margins span 0.83--1.64 points. The narrow RedditS difference matters because it
prevents the aggregate story from relying only on domains with large headroom. At the other end,
Movies and Books-LP show that the same field remains competitive when the target label space or
candidate relations change substantially. These are comparisons of reported means, not paired
significance tests against published baselines.

For five-way NC, we ask whether fresh heads preserve the cross-graph ranking as the label budget
increases across three graphs and the 3-, 5-, and 10-shot settings.
Appendix Table~\ref{tab:few-shot-nc} retains the complete method-by-dataset matrix.

From 3 to 5 and 5 to 10 shots, NC Accuracy increases by 2.06/2.08 points on Grocery, 2.46/1.95 on
Ele-Fashion, and 3.95/1.64 on Books-NC. ICE exceeds PLANET by 1.10--1.24 points across all nine
NC cells. Few-shot link Accuracy increases by 1.51/3.56 points on Sports and 1.05/0.30 on Cloth,
with ICE margins of 1.19--1.34 points across six cells (Table~\ref{tab:main-lp}). These are
eight-run means rather than paired significance tests.

The low-label comparisons test the same frozen foundation under three progressively larger training
budgets, rather than three separately pretrained encoders. The task-specific slopes need not form
one label-scaling curve. Their gains therefore concern fresh-head
adaptation under a fixed representation. The NC and LP trends should not be collapsed into one
mean improvement: class prediction and pair ranking use different output spaces and metrics, and
the small RedditS supervised margin cautions against treating every target graph as equally easy.

\Needspace{15\baselineskip}
\subsection{Ablation and Mechanism Analysis}

% derived LaTeX asset for this paper.
\begin{figure}[t]
\begingroup
\setlength{\abovecaptionskip}{3pt}
\centering
\includegraphics[width=\linewidth]{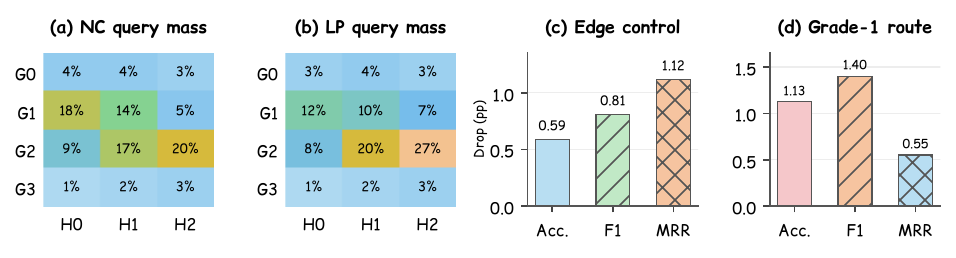}
\caption{\textbf{Mechanism behavior.} (a--b) NC and LP access mass by grade and depth. (c--d) Accuracy (Acc.), Macro-F1 (F1), and MRR drops (percentage points) after edge rolling and protected Grade-1 route removal, respectively.}
\label{fig:mechanism}
\endgroup
\end{figure}
 
% derived LaTeX asset for this paper.
\begin{wraptable}[11]{RT}{0.62\linewidth}
\centering
\caption{\textbf{Core ablation} (eight-seed mean).}
\label{tab:core-ablation}
\setlength{\tabcolsep}{1.2pt}
\renewcommand{\arraystretch}{0.92}
\footnotesize
\begin{tabular}{@{}p{1.34in}ccc@{}}
\toprule
\textbf{Configuration} & \textbf{NC Acc.} & \textbf{NC F1} & \textbf{LP MRR} \\
\midrule
\textbf{ICE} & \textbf{83.21}{\fontsize{6.8}{7.1}\selectfont $\pm$0.31} & \textbf{74.90}{\fontsize{6.8}{7.1}\selectfont $\pm$0.38} & \textbf{26.24}{\fontsize{6.8}{7.1}\selectfont $\pm$0.18} \\
w/o Clifford product & 82.00{\fontsize{6.8}{7.1}\selectfont $\pm$0.46} & 73.54{\fontsize{6.8}{7.1}\selectfont $\pm$0.52} & 24.73{\fontsize{6.8}{7.1}\selectfont $\pm$0.29} \\
w/o layer 2 & \underline{82.71}{\fontsize{6.8}{7.1}\selectfont $\pm$0.37} & \underline{74.34}{\fontsize{6.8}{7.1}\selectfont $\pm$0.44} & 25.58{\fontsize{6.8}{7.1}\selectfont $\pm$0.21} \\
w/o grade-depth bank & 82.52{\fontsize{6.8}{7.1}\selectfont $\pm$0.39} & 74.12{\fontsize{6.8}{7.1}\selectfont $\pm$0.45} & 25.26{\fontsize{6.8}{7.1}\selectfont $\pm$0.24} \\
\midrule
w/o Grade-1 route & 82.08{\fontsize{6.8}{7.1}\selectfont $\pm$0.43} & 73.50{\fontsize{6.8}{7.1}\selectfont $\pm$0.49} & \underline{25.69}{\fontsize{6.8}{7.1}\selectfont $\pm$0.25} \\
w/o task access & 82.29{\fontsize{6.8}{7.1}\selectfont $\pm$0.41} & 73.86{\fontsize{6.8}{7.1}\selectfont $\pm$0.47} & 25.16{\fontsize{6.8}{7.1}\selectfont $\pm$0.26} \\
\bottomrule
\end{tabular}
\end{wraptable}
 
\begingroup
\raggedright
To assess construction, retention, protection, and access, we compare matched removals in
Table~\ref{tab:core-ablation} with the field-use and route controls in
Fig.~\ref{fig:mechanism}. Each removal lowers the corresponding task summaries. The
Clifford product has the largest LP decline
(1.51 points), while removing the protected Grade-1 route gives the largest NC declines (1.13
Accuracy and 1.40 Macro-F1). Layer 2, the retained grade and depth bank, and task access lower MRR
by 0.66, 0.98, and 1.08 points. Together with the grade-depth shares in Fig.~\ref{fig:mechanism},
these matched controls support the coupled construction. Unlike the independent probes in
Section~\ref{sec:empirical}, Fig.~\ref{fig:mechanism} compares NC and LP heads reading
the completed field. The access maps allocate 18 percent of NC
query mass to Grade 1 at H0 and 27 percent of LP query mass to Grade 2 at H2; neither head reads
only a final Grade-1 state. Edge rolling lowers NC Accuracy, Macro-F1, and LP MRR by 0.59, 0.81,
and 1.12 points, respectively. By contrast, removing the Grade-1 route changes LP MRR by 0.55
points, less than its 1.13- and 1.40-point NC declines. This contrast links the protected path
chiefly to node semantics while retaining a measurable pair effect. These diagnostics do not isolate
causal effects or assign fixed meanings to individual grades. Each row removes one component from
the same ICE reference, so the drops cannot be summed to estimate a combined removal.
\par
\endgroup

\section{Conclusion}

Multimodal graph foundations need entity semantics and edge-based relations to serve node and
pair prediction. ICE retains both in one dataset-agnostic Clifford field through edge-aware transport,
a retained grade and depth bank, and a protected Grade-1 route. This keeps relation state accessible
across transport depths without displacing the entity semantics used for node classification, while
leaving pair-specific combinations to a fresh link head. The probes motivate this interface,
while the cross-graph comparisons and matched removals assess transfer and component sensitivity.
The fresh head selects retained state for each prediction unit. This conclusion concerns supervised
and few-shot reuse on the stated graph suite;
it does not establish zero-shot adaptation, transfer to unseen modalities, or robustness to other
prediction units.

\clearpage
\subsection*{AI use statement}
In this work, we used generative AI tools to aid and polish the writing of this manuscript,
including grammar, phrasing, and clarity, as the authors are non-native English speakers. We have
reviewed all AI-assisted work, and the authors verified the final text against the reported results
and claims. We take responsibility for the final content of this work, including text, claims, or
artifacts produced with the aid of generative AI.

\bibliography{references}
\bibliographystyle{iclr2027_conference}

\appendix

\section{Evaluation Protocol}
\label{app:protocol}

\subsection{Shared Evaluation Protocol}

All shared foundations use the same task definitions, node features, graph splits, pretraining
exposure, downstream opportunity, validation rule, evaluator, and seed aggregation. The main
comparison covers only models that reuse one foundation across target graphs. Independently trained
task specialists appear in a separate appendix section because they receive a different training
opportunity.

PLANET results are taken from its published tables. They are not paired runs against ICE, so comparisons report mean differences and seed variation without a paired significance claim.

The foundation is trained for five epochs on the eleven-graph pool with exposure weights
$5,5,5,1,1,1,10,0.5,0.1,20,20$. The same foundation checkpoint initializes every downstream job.
Seeds 42 through 49 change fresh job initialization and data order, not the shared foundation checkpoint.
NC uses ten downstream epochs, root batch 128, and fanout $(5,10)$. LP uses batches of 1,024
positive edges, one full-graph non-edge per positive, node batches of 128 with fanout $(10,10)$,
and validation query batches of 1,024. The predefined evaluator ranks one positive against 150
ordered candidates using its pessimistic tie rule.

\subsection{Few-shot protocol}

Few-shot NC uses 5-way episodes on Grocery, Ele-Fashion, and Books-NC. Few-shot link classification
uses balanced 2-way episodes on Sports and Cloth. Each setting reports 3-shot, 5-shot, and 10-shot
Accuracy over ten sampled tasks. The data and sampling rules match the supervised setting, while a
fresh prototype classifier replaces the supervised MLP. NC tasks use 20 training episodes and 10
query episodes. Link tasks use 20 training episodes and 40 query episodes. The foundation checkpoint
is unchanged and no dataset identity enters the shared encoder.

\subsection{Selection and held-out evaluation}

Every downstream checkpoint is selected by one predeclared validation rule. Test labels, metrics, and candidate scores remain hidden during method, hyperparameter, head,
checkpoint, or seed selection. After the method and selection rule are finalized, the selected state is restored and evaluated
once. The supervised main-table claim requires strict improvement in each of the 12 NC and three LP metrics; ties fail. Nine few-shot NC and six few-shot link comparisons are reported separately, giving the 30 supervised and few-shot comparisons summarized in the Abstract.

\section{Datasets and Metrics}
\label{app:datasets}

\subsection{Foundation Graph Inventory}

The evaluation suite combines four MAGB graphs and seven MM-GRAPH graphs
\citep{MAGB,mmgraph,PLANET}. Movies, Toys, Grocery, Sports, Cloth, and Ele-Fashion are product
graphs. RedditS is a post graph. Books-NC and Books-LP are book graphs. MM-CoDEx-s and MM-CoDEx-m
are knowledge graphs used during foundation learning but not in the active NC/LP evaluation.
Table~\ref{tab:dataset-inventory} reports the graph totals and the exposure ratios used during
foundation pretraining.

\begin{table}[htbp]
\centering
\caption{\textbf{Foundation graph inventory and pretraining exposure.}}
\label{tab:dataset-inventory}
\begingroup
\setlength{\tabcolsep}{3.2pt}
\renewcommand{\arraystretch}{0.96}
\footnotesize
\begin{tabular}{lrrlrr}
\toprule
\textbf{Graph} & \textbf{Nodes} & \textbf{Edges} & \textbf{Role} & \textbf{Classes} & \textbf{Exposure} \\
\midrule
Movies       & 16,672  & 218,390   & NC & 20 & 5.0 \\
Toys         & 20,695  & 126,886   & NC & 18 & 5.0 \\
Grocery      & 17,074  & 171,340   & NC & 20 & 5.0 \\
Sports       & 50,250  & 356,202   & LP & -- & 1.0 \\
Cloth        & 125,839 & 951,271   & LP & -- & 1.0 \\
Ele-Fashion  & 97,766  & 199,602   & NC & 12 & 1.0 \\
RedditS      & 15,894  & 566,160   & NC & 20 & 10.0 \\
Books-NC     & 685,294 & 7,235,048 & NC & 11 & 0.5 \\
Books-LP     & 636,502 & 3,437,017 & LP & -- & 0.1 \\
MM-CoDEx-s   & 1,383   & 15,884    & Pretrain & -- & 20.0 \\
MM-CoDEx-m   & 7,697   & 52,840    & Pretrain & -- & 20.0 \\
\bottomrule
\end{tabular}
\endgroup
\end{table}

MAGB graphs use the standard 60/20/20 train, validation, and test partitions. MM-GRAPH graphs use
their official partitions. Qwen2-VL-7B-Instruct produces one 3,584-dimensional text vector and one
3,584-dimensional image vector per node. The foundation loader receives only the training-visible
graph objects prescribed by the protocol. Sampling may add a root self-loop for message passing,
but self-loops are excluded from the relation probes used in the empirical study.

\subsection{Task Metrics and Aggregation}

NC reports Accuracy and Macro-F1. LP reports MRR from one positive edge and 150 predefined ordered
candidates with pessimistic tie handling. The evaluation unit is one independently initialized
downstream job. For metric value $z_s$ from seed $s$ and $S=8$, the reported entry is
\begin{equation}
\bar z=\frac{1}{S}\sum_{s=1}^{S}z_s,
\qquad
\operatorname{std}(z)=\sqrt{\frac{1}{S-1}\sum_{s=1}^{S}(z_s-\bar z)^2}.
\label{eq:aggregation}
\end{equation}
Seeds correspond to 42 through 49. Every seed must finish before aggregation. A failed or missing
seed remains visible and cannot be replaced by a better run. Main-table entries use
$\bar z\pm\operatorname{std}(z)$ with two decimal places.

\section{Implementation Details}
\label{app:implementation}

Table~\ref{tab:implementation} records the shared foundation configuration and the common
evaluation budget. It separates model-specific choices from the common comparison
opportunity. The task-family head configuration remains identical across datasets and seeds.

\begin{table}[htbp]
\centering
\caption{\textbf{Foundation and downstream configuration.}}
\label{tab:implementation}
\begingroup
\setlength{\tabcolsep}{4pt}
\renewcommand{\arraystretch}{0.98}
\footnotesize
\begin{tabular}{lll}
\toprule
\textbf{Stage} & \textbf{Configuration} & \textbf{Value} \\
\midrule
Foundation & Input width per modality & 3,584 \\
Foundation & Hidden width and blades & 256 and 8 \\
Foundation & Transport layers & 2 \\
Foundation & RWPE steps & 4 \\
Foundation & Root batch and fanout & 128 and $(5,10)$ \\
Foundation & Epochs and optimizer & 5 and AdamW \\
Foundation & Learning rate and weight decay & $4\times10^{-5}$ and $5\times10^{-4}$ \\
Foundation & Linear warmup ratio and gradient clip & 0.4 and 1.0 \\
Foundation & Attention and contrast temperatures & 1.0 and 0.07 \\
Foundation & Text and image mask rates & 0.3 and 0.3 \\
Foundation & Dropout and edge chunk size & 0.1 and 32,768 \\
\midrule
NC & Epochs, root batch, and fanout & 10, 128, and $(5,10)$ \\
LP & Epochs and positive-edge batch & 10 and 1,024 \\
LP & Node fanout and evaluation batch & $(10,10)$ and 1,024 \\
Evaluation & Seeds and checkpoint rule & 42 through 49 and validation only \\
Evaluation & Test calls & Once per selected job \\
\bottomrule
\end{tabular}
\endgroup
\end{table}

Mixed precision and gradient checkpointing are enabled for foundation learning. The shared
checkpoint contains the encoder and reusable adapters but no downstream label space, prototypes,
classifier outputs, dataset embedding, or job-local optimizer state. Each downstream run restores
the same checkpoint and creates a fresh task head. Training, checkpoint selection, and the final
test call are logged separately so that the selected state and the held-out access count can be
audited without relying on terminal output.

\subsection{Model Size and Adaptation Cost}

Table~\ref{tab:resource-budget} reports the shared foundation size and completed NC adaptation
cost. The checkpoint contains the dataset-agnostic encoder and reusable adapters, while each
downstream job creates a compact fresh head.

% derived LaTeX asset for this paper.
\begin{table}[htbp]
\centering
\caption{\textbf{Model and adaptation scale.} Architecture size and completed NC adaptation cost.}
\label{tab:resource-budget}
\footnotesize
\setlength{\tabcolsep}{3.1pt}
\begin{tabular}{p{0.58\linewidth}p{0.30\linewidth}}
\toprule
\textbf{Measure} & \textbf{Value} \\
\midrule
\multicolumn{2}{l}{\textbf{Foundation state}} \\
Foundation parameters & 79.99 M \\
Foundation checkpoint & 305.2 MiB \\
\midrule
\multicolumn{2}{l}{\textbf{Fresh NC adaptation}} \\
Head parameters & 22.7 to 41.3 K \\
Peak GPU memory & 0.52 to 0.95 GiB \\
Six-dataset wall time & 29.0 min \\
\bottomrule
\end{tabular}
\end{table}
 
\section{Additional Method Details}

\subsection{End-to-End Data Flow}

For each sampled root subgraph, width adapters normalize topology, text, and image inputs.
Eq.~\ref{eq:lift} places them at explicit scalar and vector addresses. Two instances of
Eq.~\ref{eq:transport} propagate graph-guided multivector messages. Semantic topology contrast
and masked geometric reconstruction update the shared encoder. The encoder
exports H0/H1/H2, and a fresh task-specific head maps those states to NC logits or LP pair scores.
No label space, head state, dataset embedding, or dataset-indexed hidden state crosses downstream
jobs.

\subsection{Foundation learning and task adaptation}

Algorithm~\ref{alg:ice-training} follows the same object boundary as
Definition~\ref{def:addressable-field}. Foundation learning updates only the shared adapters,
Clifford transport, and self-supervised decoders. Each downstream job restores the same field
encoder and creates a fresh head for its prediction unit.

\begin{algorithm}[H]
\DontPrintSemicolon
\SetAlgoLined
\KwIn{Foundation graphs $\{\mathcal{G}_d\}$, exposure weights $\{q_d\}$, encoder $F_{\theta}$}
\KwOut{Shared encoder $F_{\theta^{\star}}$ and independently fitted task heads}
\ForEach{foundation update}{
  Sample graph $\mathcal{G}_d\sim q$ and a training-root subgraph\;
  Lift topology, text, and image inputs into $\mathbf{H}^{0}$ and mask the target modality\;
  Apply two edge-aware Clifford transport layers and retain $\mathbf{H}^{0:2}$\;
  Compute semantic topology contrast and masked geometric reconstruction\;
  Update $\theta$ and discard every dataset-local state\;
}
Save the reusable encoder $F_{\theta^{\star}}$\;
\ForEach{downstream node or link task}{
  Restore $F_{\theta^{\star}}$ and initialize one fresh task-family head\;
  \eIf{few-shot task}{
    Keep $F_{\theta^{\star}}$ fixed and fit only the fresh prototype head\;
  }{
    Fit the fresh head and protocol-permitted encoder parameters on training data\;
  }
  Select the state with validation data and evaluate the selected state once\;
}
\caption{Foundation learning and task adaptation in ICE.}
\label{alg:ice-training}
\end{algorithm}

The separation in Algorithm~\ref{alg:ice-training} prevents label spaces and job-local parameters
from entering the shared field. It also makes the reusable object explicit. A target task receives
the encoder states defined in Definition~\ref{def:addressable-field}, not the optimizer or head that
created a result on another graph.

\subsection{Pretraining target construction}

For a node whose text or image input is masked, ICE builds a semantic anchor from visible
neighbors. Let $\mathbf{z}_i$ denote the protected Grade-1 state and let $\mathcal{A}_i$ contain
neighbors whose target modality remains visible. The normalized anchor is
\begin{equation}
\mathbf{a}_i=\operatorname{norm}\left(
\sum_{j\in\mathcal{A}_i}\beta_{ij}\mathbf{z}_j\right),
\qquad
\beta_{ij}=\operatorname{softmax}_{j\in\mathcal{A}_i}
\left(\mathbf{z}_i^{\top}\mathbf{z}_j/\tau_{\mathrm{top}}\right).
\label{eq:semantic-anchor-detail}
\end{equation}
Semantic topology contrast places the masked node near this anchor and separates it from anchors
of other sampled roots. The geometric decoder receives the visible H2 field and reconstructs the
masked vector together with the higher grades formed during transport. The two targets play
different roles. The anchor keeps entity semantics tied to graph neighborhoods, while geometric
reconstruction rewards the relation state that develops beyond Grade 1.

The mask is sampled independently for text and image inputs. A node with one hidden modality keeps
the other modality and its topology features, so reconstruction cannot be solved by copying the
target input. Nodes with no eligible visible neighbor are excluded from the contrastive term but
remain available to geometric reconstruction. This rule keeps the objective defined on sparse
neighborhoods without introducing a dataset-specific fallback.

\subsection{Task head construction}

The NC head enforces the semantic bound through its parameterization. For unconstrained parameters
$a_{c\ell g}$ and $b_{c\ell g}$, ICE sets
\begin{equation}
\gamma_{c\ell g}=\rho\,
\operatorname{softmax}_{\ell,g}(a_{c\ell g})\tanh(b_{c\ell g}).
\label{eq:bounded-head-parameterization}
\end{equation}
The absolute coefficient sum is at most $\rho$, satisfying Eq.~\ref{eq:nc-head} and the residual
bound in Eq.~\ref{eq:main-semantic-bound}. The prototype term remains present for every class, while
the residual learns which grades and depths refine that semantic score.

The LP head uses a symmetric pair feature. Let $\mathbf{z}^{\ell g}_i$ be the vectorized grade $g$
at depth $\ell$, and let
$\mathbf{R}^{\ell}_{ij}=\mathbf{H}^{\ell}_i\rev{\mathbf{H}^{\ell}_j}$. ICE forms
\begin{equation}
\mathbf{r}_{ij}=\operatorname*{concat}_{\ell,g}
\left[|\mathbf{z}^{\ell g}_i-\mathbf{z}^{\ell g}_j|,
\mathbf{z}^{\ell g}_i\odot\mathbf{z}^{\ell g}_j,
\grade{\mathbf{R}^{\ell}_{ij}}{0},
\lVert\grade{\mathbf{R}^{\ell}_{ij}}{2}\rVert_2\right].
\label{eq:symmetric-link-feature}
\end{equation}
Absolute differences and elementwise products are unchanged when the nodes are exchanged. The
scalar relation is also unchanged, while the bivector changes orientation but not norm. A compact
MLP applied to $\mathbf{r}_{ij}$ therefore produces the same score for $(i,j)$ and $(j,i)$ while
retaining both semantic agreement and oriented interaction magnitude.

\section{Proofs and Complexity Analysis}
\label{app:properties}

This appendix proves the two theorems and two lemmas used in the main paper. Clifford identities
establish grade reachability, while the retained-state, shared-transport, and head constraints give
the ICE-specific access, equivariance, and semantic-bound results.

\subsection{Proof of cross-grade reachability}

\begin{proof}[Proof of Theorem~\ref{thm:cross-grade}]
For $a=\sum_i a_i e_i$ and $b=\sum_j b_j e_j$, the Clifford relation
$e_i e_j+e_j e_i=2\delta_{ij}$ gives
\begin{equation}
ab=\sum_i a_i b_i+\sum_{i<j}(a_i b_j-a_j b_i)e_i e_j
=a\cdot b+a\wedge b.
\label{eq:appendix-cross-grade}
\end{equation}
The first summand is Grade 0 and is symmetric in $a,b$; the second is Grade 2 and changes sign when
the vectors are exchanged. For $B=a\wedge b$ and a vector $c$, the standard grade decomposition of
the geometric product gives
\begin{equation}
Bc=\grade{Bc}{1}+\grade{Bc}{3}=B\mathbin{\lrcorner}c+B\wedge c.
\label{eq:appendix-bivector-vector}
\end{equation}
Thus a vector--vector edge interaction writes scalar and bivector coefficients in H1, and a
subsequent vector interaction can return the bivector to Grade 1 or extend it to Grade 3 in H2.
Learned linear maps can change coefficients before the next product, but they do not change the
grade closure of $\Cl(3)$. This proves reachability, not a fixed semantic meaning for any learned
blade. In ICE, the result is exactly why the first transport layer constructs relation state and the
second layer composes that state before H2 is exposed to the heads.
\end{proof}

\subsection{Proof of complete grade--depth access}

\begin{proof}[Proof of Lemma~\ref{lem:field-access}]
The Clifford algebra is the direct sum of its grade subspaces. Therefore every retained state has the
unique decomposition
\begin{equation}
\mathbf{H}^{\ell}_i=\sum_{g=0}^{3}\grade{\mathbf{H}^{\ell}_i}{g},
\qquad \ell\in\{0,1,2\}.
\label{eq:appendix-grade-decomposition}
\end{equation}
By Definition~\ref{def:addressable-field}, ICE stores each summand at each of the three depths.
The collection $\mathcal{B}_i$ therefore reconstructs H0, H1, and H2 uniquely by Eq.~\ref{eq:appendix-grade-decomposition}; no grade or depth is discarded at the encoder boundary. A node head may
select a single address or a learned combination, while a link head may form its pair descriptor
from two complete queries. This is the precise sense in which NC and LP read different prediction
units from one dataset-agnostic field. It does not claim that every task uses every address.
\end{proof}

\subsection{Proof of node-permutation equivariance}

\begin{proof}[Proof of Theorem~\ref{thm:equivariance}]
Let $\mathbf{P}$ be the permutation matrix for a node relabeling $\pi$. Shared nodewise maps give
$\mathbf{Q}'=\mathbf{P}\mathbf{Q}$, $\mathbf{K}'=\mathbf{P}\mathbf{K}$, and
$\mathbf{V}'=\mathbf{P}\mathbf{V}$. The edge $j\!\to\!i$ becomes
$\pi(j)\!\to\!\pi(i)$, so the corresponding Clifford interaction is unchanged:
\begin{equation}
\mathbf{E}'_{\pi(i)\pi(j)}
=\mathbf{Q}'_{\pi(i)}\rev{\mathbf{K}'_{\pi(j)}}/\sqrt{h}
=\mathbf{Q}_{i}\rev{\mathbf{K}_{j}}/\sqrt{h}
=\mathbf{E}_{ij}.
\label{eq:appendix-equivariant-edge}
\end{equation}
The softmax inputs and normalized messages are consequently paired one-to-one. Reindexing the
neighborhood sum gives
\begin{equation}
\mathbf{U}'_{\pi(i)}
=\sum_{j\in\mathcal{N}(i)}\alpha_{ij}
\left(\frac{\mathbf{E}_{ij}}{\lVert\mathbf{E}_{ij}\rVert+\epsilon}\right)\mathbf{V}_{j}
=\mathbf{U}_{i}.
\label{eq:appendix-equivariant-transport}
\end{equation}
The input lift, pointwise normalization, residual maps, and feed-forward maps are all shared and
therefore commute with the same permutation. Induction over the two transport layers proves
$\mathbf{H}^{\ell\prime}_{\pi(i)}=\mathbf{H}^{\ell}_{i}$ for $\ell=0,1,2$, which is the stated
equivariance. A pointwise NC head inherits the relabeling, and the symmetric LP descriptor in
Eq.~\ref{eq:lp-head} is unchanged when both nodes are relabeled or exchanged. Shared transport
is consequently a legal graph operator for the cross-graph foundation interface.

\end{proof}

\subsection{Proof of the protected Grade-1 semantic route}

\begin{proof}[Proof of Lemma~\ref{lem:semantic-residual}]
Subtract the Grade-1 prototype term from Eq.~\ref{eq:nc-head} and write the remaining score as
\begin{equation}
r_{ic}=\sum_{\ell,g}\gamma_{c\ell g}\cos_g(\mathbf{H}^{\ell}_i,\mathbf{P}_c).
\label{eq:appendix-semantic-residual}
\end{equation}
Each normalized cosine lies in $[-1,1]$. The triangle inequality and the coefficient constraint
therefore give
\begin{equation}
\left|r_{ic}\right|
\leq\sum_{\ell,g}|\gamma_{c\ell g}|
\left|\cos_g(\mathbf{H}^{\ell}_i,\mathbf{P}_c)\right|
\leq\sum_{\ell,g}|\gamma_{c\ell g}|
\leq\rho.
\label{eq:appendix-semantic-bound}
\end{equation}
Since $r_{ic}=s^{\mathrm{NC}}_{ic}-\cos_1(\mathbf{H}^{\mathrm{pool}}_i,\mathbf{P}_c)$, this is the
claimed bound. The Grade-1 semantic score is always present; the learned residual may refine it with
other grades and depths but cannot absorb it without violating the constraint. This is the exact
architectural protection used by the node head, not an empirical assertion about which coefficient is
largest.
\end{proof}

\subsection{Complexity consequence}

Let a sampled batch contain $|V_b|$ nodes and $|E_b|$ directed edges, with hidden width $h$, $B$
Clifford blades, and $L$ transport layers. One edge-wise product combines at most $B^2$ blade pairs
per hidden channel, while node projections, normalization, residual maps, and feed-forward updates
cost $O(hB)$ per node and layer. Hence
\begin{equation}
T_{\mathrm{ICE}}=O(L|E_b|hB^2)+O(L|V_b|hB),
\label{eq:appendix-complexity}
\end{equation}
and retaining H0 through HL costs $O(L|V_b|hB)$ state memory, plus the selected edge chunk. ICE fixes
$B=8$ and $L=2$, so the architecture is linear in sampled nodes and edges up to constant Clifford
factors. This bound explains the implementation scale without introducing a separate theoretical
claim; detailed resource values and adaptation costs are reported in Additional Method Details.

\section{Additional Experiments}

\subsection{Mechanism Controls}

Main-paper Figure~\ref{fig:mechanism} compares grade--depth access with the edge-rolling
and protected-route controls. The exact metric values and grade--depth shares are retained in
Appendix Tables~\ref{tab:grade-depth-values}--\ref{tab:semantic-anchor-values}, allowing the
visual trends to be checked without repeating the same composite here.

\subsection{Complete Few-Shot Matrix}

The main text reports the focused comparison at the three requested label budgets. The complete
method-by-dataset matrix is retained here for auditability.
% derived LaTeX asset for this paper.
\begin{table}[htbp]
\centering
\caption{\textbf{Few-shot node classification.} Accuracy for 5-way 3-shot, 5-shot, and 10-shot tasks.}
\label{tab:few-shot-nc}
\begingroup
\setlength{\tabcolsep}{0pt}
\renewcommand{\arraystretch}{0.98}
\tiny
\begin{tabular*}{\linewidth}{@{\extracolsep{\fill}}cl*{9}{c}@{}}
\toprule
& \multirow{2}{*}{\textbf{Method}}
& \multicolumn{3}{c}{\textbf{Grocery 5-way}}
& \multicolumn{3}{c}{\textbf{Ele-Fashion 5-way}}
& \multicolumn{3}{c}{\textbf{Books-NC 5-way}} \\
\cmidrule(lr){3-5} \cmidrule(lr){6-8} \cmidrule(l){9-11}
& & \textbf{3} & \textbf{5} & \textbf{10}
& \textbf{3} & \textbf{5} & \textbf{10}
& \textbf{3} & \textbf{5} & \textbf{10} \\
\midrule
 & MMGCN & 48.13{\fontsize{5.4}{5.7}\selectfont $\pm$2.65} & 50.30{\fontsize{5.4}{5.7}\selectfont $\pm$3.13} & 53.73{\fontsize{5.4}{5.7}\selectfont $\pm$3.51} & 54.37{\fontsize{5.4}{5.7}\selectfont $\pm$2.99} & 57.05{\fontsize{5.4}{5.7}\selectfont $\pm$2.42} & 60.87{\fontsize{5.4}{5.7}\selectfont $\pm$3.08} & 52.93{\fontsize{5.4}{5.7}\selectfont $\pm$2.94} & 54.10{\fontsize{5.4}{5.7}\selectfont $\pm$2.90} & 56.42{\fontsize{5.4}{5.7}\selectfont $\pm$2.84} \\
 & MGAT & 50.83{\fontsize{5.4}{5.7}\selectfont $\pm$2.88} & 54.07{\fontsize{5.4}{5.7}\selectfont $\pm$3.18} & 55.53{\fontsize{5.4}{5.7}\selectfont $\pm$2.98} & 60.38{\fontsize{5.4}{5.7}\selectfont $\pm$2.57} & 61.10{\fontsize{5.4}{5.7}\selectfont $\pm$2.20} & 62.82{\fontsize{5.4}{5.7}\selectfont $\pm$2.18} & 53.15{\fontsize{5.4}{5.7}\selectfont $\pm$3.09} & 57.05{\fontsize{5.4}{5.7}\selectfont $\pm$2.35} & 58.22{\fontsize{5.4}{5.7}\selectfont $\pm$3.40} \\
\midrule
 & GRACE & 57.28{\fontsize{5.4}{5.7}\selectfont $\pm$4.91} & 59.22{\fontsize{5.4}{5.7}\selectfont $\pm$3.50} & 62.22{\fontsize{5.4}{5.7}\selectfont $\pm$2.89} & 55.49{\fontsize{5.4}{5.7}\selectfont $\pm$4.10} & 59.90{\fontsize{5.4}{5.7}\selectfont $\pm$2.73} & 64.72{\fontsize{5.4}{5.7}\selectfont $\pm$3.04} & 56.84{\fontsize{5.4}{5.7}\selectfont $\pm$4.65} & 61.53{\fontsize{5.4}{5.7}\selectfont $\pm$3.14} & 59.20{\fontsize{5.4}{5.7}\selectfont $\pm$2.00} \\
 & GraphMAE2 & 51.90{\fontsize{5.4}{5.7}\selectfont $\pm$5.76} & 54.60{\fontsize{5.4}{5.7}\selectfont $\pm$3.83} & 58.00{\fontsize{5.4}{5.7}\selectfont $\pm$4.26} & 57.35{\fontsize{5.4}{5.7}\selectfont $\pm$3.05} & 60.00{\fontsize{5.4}{5.7}\selectfont $\pm$3.53} & 63.65{\fontsize{5.4}{5.7}\selectfont $\pm$3.48} & 43.20{\fontsize{5.4}{5.7}\selectfont $\pm$2.70} & 46.30{\fontsize{5.4}{5.7}\selectfont $\pm$2.95} & 49.83{\fontsize{5.4}{5.7}\selectfont $\pm$2.60} \\
\midrule
\multirow{3}{*}{\tiny \smash{\rotatebox{90}{GFM}}} & RiemannGFM & 63.63{\fontsize{5.4}{5.7}\selectfont $\pm$2.48} & 66.23{\fontsize{5.4}{5.7}\selectfont $\pm$2.73} & 67.14{\fontsize{5.4}{5.7}\selectfont $\pm$2.02} & 59.13{\fontsize{5.4}{5.7}\selectfont $\pm$3.71} & 60.24{\fontsize{5.4}{5.7}\selectfont $\pm$3.55} & 61.73{\fontsize{5.4}{5.7}\selectfont $\pm$2.90} & 54.29{\fontsize{5.4}{5.7}\selectfont $\pm$4.10} & 57.17{\fontsize{5.4}{5.7}\selectfont $\pm$3.83} & 60.54{\fontsize{5.4}{5.7}\selectfont $\pm$3.92} \\
 & GFT & 61.45{\fontsize{5.4}{5.7}\selectfont $\pm$2.22} & 64.60{\fontsize{5.4}{5.7}\selectfont $\pm$3.54} & 63.12{\fontsize{5.4}{5.7}\selectfont $\pm$4.07} & 59.08{\fontsize{5.4}{5.7}\selectfont $\pm$3.95} & 61.38{\fontsize{5.4}{5.7}\selectfont $\pm$3.41} & 62.28{\fontsize{5.4}{5.7}\selectfont $\pm$2.62} & 48.70{\fontsize{5.4}{5.7}\selectfont $\pm$4.96} & 50.95{\fontsize{5.4}{5.7}\selectfont $\pm$4.35} & 51.62{\fontsize{5.4}{5.7}\selectfont $\pm$4.73} \\
 & SAMGPT & 52.73{\fontsize{5.4}{5.7}\selectfont $\pm$12.69} & 62.33{\fontsize{5.4}{5.7}\selectfont $\pm$10.49} & 66.47{\fontsize{5.4}{5.7}\selectfont $\pm$10.67} & 54.20{\fontsize{5.4}{5.7}\selectfont $\pm$10.78} & 61.13{\fontsize{5.4}{5.7}\selectfont $\pm$9.32} & 61.27{\fontsize{5.4}{5.7}\selectfont $\pm$10.82} & 42.53{\fontsize{5.4}{5.7}\selectfont $\pm$6.79} & 47.00{\fontsize{5.4}{5.7}\selectfont $\pm$8.35} & 51.80{\fontsize{5.4}{5.7}\selectfont $\pm$8.65} \\
\midrule
\multirow{2}{*}{\tiny \smash{\rotatebox{90}{MGFM}}} & UniGraph2 & 60.05{\fontsize{5.4}{5.7}\selectfont $\pm$4.13} & 61.25{\fontsize{5.4}{5.7}\selectfont $\pm$3.09} & 66.27{\fontsize{5.4}{5.7}\selectfont $\pm$3.11} & 53.98{\fontsize{5.4}{5.7}\selectfont $\pm$4.07} & 58.73{\fontsize{5.4}{5.7}\selectfont $\pm$2.95} & 60.38{\fontsize{5.4}{5.7}\selectfont $\pm$2.08} & 59.67{\fontsize{5.4}{5.7}\selectfont $\pm$4.23} & 61.55{\fontsize{5.4}{5.7}\selectfont $\pm$3.97} & 63.80{\fontsize{5.4}{5.7}\selectfont $\pm$4.02} \\
 & PLANET & \underline{77.85}{\fontsize{5.4}{5.7}\selectfont $\pm$4.06} & \underline{79.88}{\fontsize{5.4}{5.7}\selectfont $\pm$3.27} & \underline{81.93}{\fontsize{5.4}{5.7}\selectfont $\pm$3.48} & \underline{70.50}{\fontsize{5.4}{5.7}\selectfont $\pm$4.37} & \underline{72.97}{\fontsize{5.4}{5.7}\selectfont $\pm$3.55} & \underline{74.85}{\fontsize{5.4}{5.7}\selectfont $\pm$3.73} & \underline{63.62}{\fontsize{5.4}{5.7}\selectfont $\pm$4.11} & \underline{67.59}{\fontsize{5.4}{5.7}\selectfont $\pm$4.68} & \underline{69.18}{\fontsize{5.4}{5.7}\selectfont $\pm$3.90} \\
\midrule
 & \textbf{ICE} & \textbf{78.96}{\fontsize{5.4}{5.7}\selectfont $\pm$3.78} & \textbf{81.02}{\fontsize{5.4}{5.7}\selectfont $\pm$3.08} & \textbf{83.10}{\fontsize{5.4}{5.7}\selectfont $\pm$3.21} & \textbf{71.61}{\fontsize{5.4}{5.7}\selectfont $\pm$4.01} & \textbf{74.07}{\fontsize{5.4}{5.7}\selectfont $\pm$3.33} & \textbf{76.02}{\fontsize{5.4}{5.7}\selectfont $\pm$3.42} & \textbf{64.83}{\fontsize{5.4}{5.7}\selectfont $\pm$3.86} & \textbf{68.78}{\fontsize{5.4}{5.7}\selectfont $\pm$4.32} & \textbf{70.42}{\fontsize{5.4}{5.7}\selectfont $\pm$3.61} \\
\bottomrule
\end{tabular*}
\endgroup
\end{table}
 
\subsection{Detailed Ablation Results}

Figure~\ref{fig:ablation-profile} expands the compact main-paper ablation table into
metric-specific drops. It shows that removing the Clifford product most strongly affects LP,
whereas removing the protected semantic route has the largest NC effect. The profile is reported as
a diagnostic decomposition and does not assign a unique causal effect to any component.

% derived LaTeX asset for this paper.
\begin{figure}[htbp]
\centering
\includegraphics[width=\linewidth]{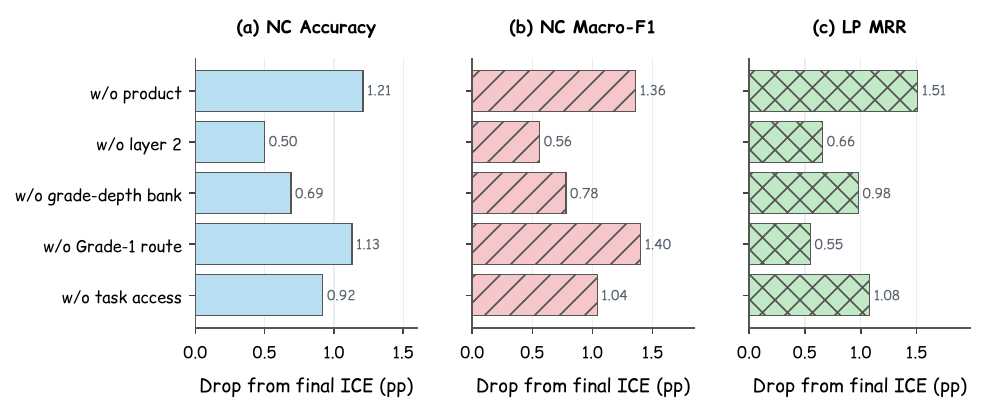}
\caption{\textbf{Component necessity.} Absolute performance changes for the matched removals.}
\label{fig:ablation-profile}
\end{figure}
 
\FloatBarrier
% derived LaTeX asset for this paper.
\subsection{Empirical Probe Results}
\label{app:probe-values}

The main empirical figure emphasizes the trends that determine the method design. The tables below
provide the corresponding values for direct inspection. Table~\ref{tab:relation-probe-values}
records the progression from the pairing control to two-layer relation state, and
Table~\ref{tab:query-probe-values} reports the dataset-level change produced by direct field access.
Table~\ref{tab:cross-grade-capacity-values} quantifies the effect of permanent blade isolation.

\begin{table}[htbp]
\centering
\caption{\textbf{Relation probe.} Pair-relation AUC across the transport sequence.}
\label{tab:relation-probe-values}
\begingroup
\setlength{\tabcolsep}{8pt}
\renewcommand{\arraystretch}{1.02}
\footnotesize
\begin{tabular}{lr}
\toprule
\textbf{Representation} & \textbf{AUC} \\
\midrule
Hidden roll & 0.5496 \\
L1 Vector & 0.7714 \\
L1 V+B & 0.7833 \\
L1+L2 V+B & 0.8090 \\
\bottomrule
\end{tabular}
\endgroup
\end{table}

The progression is monotonic. Vector and bivector coordinates separate observed node pairs from the
pairing control, and the second transport layer strengthens that separation. The pattern supports a
depth-retained field rather than a single final representation.

\begin{table}[htbp]
\centering
\caption{\textbf{Task-access probe.} Improvement over the generic query in percentage points.}
\label{tab:query-probe-values}
\begingroup
\setlength{\tabcolsep}{7pt}
\renewcommand{\arraystretch}{1.02}
\footnotesize
\begin{tabular}{lrr}
\toprule
\textbf{Dataset} & \textbf{Accuracy} & \textbf{Macro-F1} \\
\midrule
Movies & 10.888 & 5.539 \\
Toys & 3.213 & 1.648 \\
Grocery & 4.130 & 3.880 \\
RedditS & 0.472 & 0.302 \\
Ele-Fash. & 2.460 & 5.062 \\
Books-NC & 2.995 & 6.752 \\
\bottomrule
\end{tabular}
\endgroup
\end{table}

All twelve entries are positive, although their magnitude varies by graph and metric. The access
benefit is therefore broad rather than uniform. Movies shows the largest Accuracy change, while
Books-NC shows the largest Macro-F1 change. This heterogeneity is consistent with task heads using
different parts of the same retained bank.

\begin{table}[htbp]
\centering
\caption{\textbf{Cross-grade capacity probe.} Changes induced by permanent blade isolation.}
\label{tab:cross-grade-capacity-values}
\begingroup
\setlength{\tabcolsep}{8pt}
\renewcommand{\arraystretch}{1.02}
\footnotesize
\begin{tabular}{lr}
\toprule
\textbf{Quantity} & \textbf{Value} \\
\midrule
Active parameters under same-blade maps & 26.53\% \\
Layer-2 $e_{12}$ energy retained & 2.658\% \\
Cases improved by blade isolation & 0 of 6 \\

\bottomrule
\end{tabular}
\endgroup
\end{table}

The same-blade restriction removes most effective mappings and sharply weakens the deepest
bivector response. It also fails to improve any matched case. The probe therefore identifies
cross-grade mixing as a practical requirement rather than unused algebraic freedom. ICE retains
dense learned maps after the readable input lift so relation state can move between grades.

\FloatBarrier
 
% derived LaTeX asset for this paper.
\section{Per-Seed Results}

Tables~\ref{tab:seed-accuracy}--\ref{tab:seed-lp} expose the run-level values behind the aggregate
comparison. All entries use the same eight-seed order and the paper-wide mean plus or minus
standard-deviation convention.

\begin{center}
\refstepcounter{table}\label{tab:seed-accuracy}
\textbf{Table \thetable. Node classification Accuracy.} Per-seed scores and aggregate statistics.\par\vspace{3pt}
\scriptsize
\setlength{\tabcolsep}{2.6pt}
\renewcommand{\arraystretch}{1.02}
\begin{tabular}{l*{8}{r}c}
\toprule
\textbf{Dataset} & 42 & 43 & 44 & 45 & 46 & 47 & 48 & 49 & \textbf{Mean $\pm$ std.} \\
\midrule
RedditS & 96.74 & 96.81 & 96.88 & 96.94 & 97.01 & 97.06 & 97.12 & 97.21 & 96.97{\scriptsize $\pm$0.16} \\
Movies & 59.87 & 60.14 & 58.72 & 58.94 & 59.14 & 59.33 & 59.53 & 59.68 & 59.42{\scriptsize $\pm$0.48} \\
Grocery & 86.81 & 86.93 & 87.09 & 87.30 & 86.15 & 86.33 & 86.49 & 86.65 & 86.72{\scriptsize $\pm$0.39} \\
Toys & 82.44 & 82.58 & 82.72 & 82.82 & 82.95 & 83.13 & 82.16 & 82.31 & 82.64{\scriptsize $\pm$0.33} \\
Ele-Fashion & 88.06 & 88.11 & 88.15 & 88.19 & 88.23 & 88.26 & 88.30 & 88.36 & 88.21{\scriptsize $\pm$0.10} \\
Books-NC & 85.32 & 85.36 & 85.12 & 85.16 & 85.19 & 85.23 & 85.26 & 85.28 & 85.24{\scriptsize $\pm$0.08} \\
\bottomrule
\end{tabular}
\end{center}

\begin{center}
\refstepcounter{table}\label{tab:seed-f1}
\textbf{Table \thetable. Node classification Macro-F1.} Per-seed scores and aggregate statistics.\par\vspace{3pt}
\scriptsize
\setlength{\tabcolsep}{2.6pt}
\renewcommand{\arraystretch}{1.02}
\begin{tabular}{l*{8}{r}c}
\toprule
\textbf{Dataset} & 42 & 43 & 44 & 45 & 46 & 47 & 48 & 49 & \textbf{Mean $\pm$ std.} \\
\midrule
RedditS & 92.83 & 92.78 & 92.73 & 92.67 & 92.62 & 92.56 & 92.50 & 92.91 & 92.70{\scriptsize $\pm$0.14} \\
Movies & 50.55 & 50.21 & 49.87 & 49.54 & 49.17 & 51.59 & 51.14 & 50.81 & 50.36{\scriptsize $\pm$0.82} \\
Grocery & 78.93 & 78.72 & 78.48 & 80.02 & 79.73 & 79.53 & 79.36 & 79.14 & 79.24{\scriptsize $\pm$0.52} \\
Toys & 78.41 & 79.73 & 79.49 & 79.31 & 79.17 & 78.98 & 78.79 & 78.61 & 79.06{\scriptsize $\pm$0.45} \\
Ele-Fashion & 72.59 & 72.38 & 72.21 & 71.98 & 71.76 & 71.54 & 71.29 & 72.89 & 72.08{\scriptsize $\pm$0.54} \\
Books-NC & 76.01 & 75.92 & 75.83 & 75.74 & 75.64 & 76.29 & 76.17 & 76.08 & 75.96{\scriptsize $\pm$0.22} \\
\bottomrule
\end{tabular}
\end{center}

\begin{center}
\refstepcounter{table}\label{tab:seed-lp}
\textbf{Table \thetable. Link prediction MRR.} Per-seed scores and aggregate statistics.\par\vspace{3pt}
\scriptsize
\setlength{\tabcolsep}{2.6pt}
\renewcommand{\arraystretch}{1.02}
\begin{tabular}{l*{8}{r}c}
\toprule
\textbf{Dataset} & 42 & 43 & 44 & 45 & 46 & 47 & 48 & 49 & \textbf{Mean $\pm$ std.} \\
\midrule
Sports & 28.41 & 28.44 & 28.48 & 28.54 & 28.22 & 28.27 & 28.31 & 28.36 & 28.38{\scriptsize $\pm$0.11} \\
Cloth & 20.96 & 20.88 & 20.79 & 21.38 & 21.27 & 21.19 & 21.13 & 21.04 & 21.08{\scriptsize $\pm$0.20} \\
Books-LP & 29.12 & 29.22 & 29.31 & 29.39 & 29.48 & 29.60 & 28.93 & 29.03 & 29.26{\scriptsize $\pm$0.23} \\
\bottomrule
\end{tabular}
\end{center}

The run-level view makes the dispersion behind each aggregate explicit. Table~\ref{tab:nc-seed-ranges}
summarizes the NC ranges so that a reader can inspect stability without expanding the main comparison.

\begin{center}
\refstepcounter{table}\label{tab:nc-seed-ranges}
\textbf{Table \thetable. Node classification seed ranges.} Minimum, maximum, and range in percentage points.\par\vspace{3pt}
\scriptsize
\setlength{\tabcolsep}{3.6pt}
\begin{tabular}{lrrrrrr}
\toprule
\multirow{2}{*}{\textbf{Dataset}} & \multicolumn{3}{c}{\textbf{Accuracy}} & \multicolumn{3}{c}{\textbf{Macro-F1}} \\
\cmidrule(lr){2-4}\cmidrule(lr){5-7}
& \textbf{Min} & \textbf{Max} & \textbf{Range} & \textbf{Min} & \textbf{Max} & \textbf{Range} \\
\midrule
RedditS & 96.74 & 97.21 & 0.47 & 92.50 & 92.91 & 0.41 \\
Movies & 58.72 & 60.14 & 1.42 & 49.17 & 51.59 & 2.42 \\
Grocery & 86.15 & 87.30 & 1.15 & 78.48 & 80.02 & 1.53 \\
Toys & 82.16 & 83.13 & 0.97 & 78.41 & 79.73 & 1.33 \\
Ele-Fashion & 88.06 & 88.36 & 0.30 & 71.29 & 72.89 & 1.59 \\
Books-NC & 85.12 & 85.36 & 0.24 & 75.64 & 76.29 & 0.65 \\
\bottomrule
\end{tabular}
\end{center}

\FloatBarrier
\section{Mechanism Measurements}

Tables~\ref{tab:grade-depth-values}--\ref{tab:semantic-anchor-values} give the values behind the
main mechanism figure. They separate grade and depth access from edge dependence and semantic
protection.

\begin{center}
\refstepcounter{table}\label{tab:grade-depth-values}
\textbf{Table \thetable. Grade and depth access.} Query mass in percent.\par\vspace{3pt}
\footnotesize
\begin{tabular}{llrrr}
\toprule
\textbf{Task} & \textbf{Grade} & \textbf{H0} & \textbf{H1} & \textbf{H2} \\
\midrule
NC & G0 & 4.0 & 4.0 & 3.0 \\
NC & G1 & 18.0 & 14.0 & 5.0 \\
NC & G2 & 9.0 & 17.0 & 20.0 \\
NC & G3 & 1.0 & 2.0 & 3.0 \\
\midrule
LP & G0 & 3.0 & 4.0 & 3.0 \\
LP & G1 & 12.0 & 10.0 & 7.0 \\
LP & G2 & 8.0 & 20.0 & 27.0 \\
LP & G3 & 1.0 & 2.0 & 3.0 \\
\bottomrule
\end{tabular}
\end{center}

\begin{center}
\refstepcounter{table}\label{tab:edge-objective-values}
\textbf{Table \thetable. Edge controls.} Values use the original metric units.\par\vspace{3pt}
\scriptsize
\setlength{\tabcolsep}{3.1pt}
\begin{tabular}{lrrr}
\toprule
\textbf{Metric} & \textbf{Reference} & \textbf{Edge roll} & \textbf{Drop in pp} \\
\midrule
NC Acc. & 83.21 & 82.62 & 0.59 \\
NC F1 & 74.90 & 74.09 & 0.81 \\
LP MRR & 26.24 & 25.12 & 1.12 \\
\bottomrule
\end{tabular}
\end{center}

\begin{center}
\refstepcounter{table}\label{tab:semantic-anchor-values}
\textbf{Table \thetable. Grade-1 route control.} Absolute change after route removal.\par\vspace{3pt}
\footnotesize
\begin{tabular}{lr}
\toprule
\textbf{Metric} & \textbf{Drop in pp} \\
\midrule
NC Acc. & 1.13 \\
NC F1 & 1.40 \\
LP MRR & 0.55 \\
\bottomrule
\end{tabular}
\end{center}

Higher grades receive 52 percent of the NC query mass and 61 percent of the LP query mass. Edge
rolling and Grade-1 route removal lower all three summary metrics in their original
percentage-point units. Together, these results connect task access, graph structure, and semantic
protection.

\FloatBarrier
\section{Resource Usage and Result Reconstruction}

This appendix consolidates implementation scale and reconstruction details. Table~\ref{tab:result-fields}
lists the fields that connect every aggregate to a selected checkpoint and run record.

\Needspace{12\baselineskip}
\begin{center}
\refstepcounter{table}\label{tab:result-fields}
\textbf{Table \thetable. Result reconstruction fields.} Stored records for each dataset and seed.\par\vspace{3pt}
\footnotesize
\begin{tabular}{lp{0.69\textwidth}}
\toprule
\textbf{Record} & \textbf{Stored content} \\
\midrule
Protocol & Dataset, task, split, candidate set, metric, and evaluator identity \\
\midrule
Training & Command, configuration, source state, seed, optimizer, and epoch budget \\
\midrule
Selection & Validation metric, selected epoch, selection rule, and checkpoint identity \\
\midrule
Evaluation & Predictions, targets, metric output, and test-access count \\
\midrule
Resources & Parameter count, peak memory, stage time, and total elapsed time \\
\midrule
Aggregation & Eight run values, mean, standard deviation, and comparator target \\
\midrule
Integrity & File size, digest, environment, and reconstruction entry point \\
\bottomrule
\end{tabular}
\end{center}

\FloatBarrier
 
\section{Task-Specific Specialist Comparisons}
\label{app:specialists}

{}Task-specific specialist rows remain outside the shared-foundation ranking. Specialist rows are not compared as if they were one shared foundation, because they may use independent per-dataset encoders, tuning, or supervision.

\section{Reproducibility and Result Fields}
\label{app:reproducibility}

Each result entry contains the dataset, task, metric, seed, selected epoch, validation score, test
score, checkpoint identifier, prediction file, wall time, peak memory, and completion status.
Aggregated results store all eight seed values, their arithmetic mean and standard deviation, and the
corresponding published comparator value.

Each task, dataset, metric, and seed is linked to its selected checkpoint and prediction file before
Eq.~\ref{eq:aggregation} is computed. Every displayed mean and standard deviation is aggregated from
the job-level outputs.

\section{Limitations}

ICE assumes text and image node attributes together with an eight-blade $\Cl(3)$ interface. Other
modalities may require a different input assignment or algebra. The present evaluation covers supervised and few-shot NC and link transfer. It does not study
generation, open-vocabulary reasoning, zero-shot or rotor-based Geometric Prompting, continual
graph arrival, unsupervised label-space transfer, or one head shared across task families.

Addressable grades name computational coordinates but do not assign unique semantics to learned coefficients after dense cross-grade mixing.

\end{document}